\pdfoutput=1

\documentclass[11pt]{article}

\usepackage[final]{acl}

\usepackage{times}
\usepackage{latexsym}

\usepackage[T1]{fontenc}

\usepackage[utf8]{inputenc}

\usepackage{microtype}

\usepackage{inconsolata}

\usepackage{graphicx}

\usepackage{multirow}
\usepackage{colortbl}
\usepackage{booktabs}
\usepackage{hhline}
\usepackage{color}
\usepackage{tabularray}
\usepackage{subcaption}
\usepackage{float} 
\usepackage{stfloats}

\usepackage{needspace}
\usepackage{amsmath}
\usepackage{amssymb}
\usepackage{mathtools}
\usepackage{amsthm}
\usepackage{bbm}

\usepackage{tcolorbox}
\tcbuselibrary{most}
\usepackage{lipsum}

\theoremstyle{plain}
\newtheorem{theorem}{Theorem}[section]

\newtheorem{corollary}[theorem]{Corollary}
\theoremstyle{definition}

\theoremstyle{remark}

\usepackage{enumitem}
\usepackage{listings}

\definecolor{codegreen}{rgb}{0,0.6,0}
\definecolor{codegray}{rgb}{0.5,0.5,0.5}
\definecolor{codepurple}{rgb}{0.58,0,0.82}
\definecolor{backcolour}{rgb}{0.95,0.95,0.92}
\lstdefinestyle{mystyle}{
    backgroundcolor=\color{backcolour},   
    commentstyle=\color{codegreen},
    keywordstyle=\color{magenta},
    numberstyle=\tiny\color{codegray},
    stringstyle=\color{codepurple},
    basicstyle=\ttfamily\footnotesize,
    breaklines=true, 
    breakindent=0pt,
    captionpos=b,                    
    keepspaces=true,                 
    numbersep=5pt,                  
    showspaces=false,                
    showstringspaces=false,
    frame=single
}

\usepackage{fontawesome}
\usepackage{comment}
\newcommand{\wwx}[1]{\textcolor{teal}{Wenxiao: #1}}
\newcommand{\cyz}[1]{\textcolor{cyan}{yize: #1}}

\newcommand{\header}[1]{\noindent\textbf{#1}.}
\title{DyePack: Provably Flagging Test Set Contamination in LLMs Using Backdoors}

\author{
Yize Cheng\thanks{Equal contribution}, Wenxiao Wang\footnotemark[1], Mazda Moayeri, Soheil Feizi \\
University of Maryland, College Park \\
\texttt{\{yzcheng, wwx, mmoayeri\}@umd.edu, sfeizi@cs.umd.edu} \\
\faGithub~Project: \url{https://github.com/chengez/DyePack}
}

\begin{document}
\maketitle

\begin{abstract}
Open benchmarks are essential for evaluating and advancing large language models, offering reproducibility and transparency. 
However, their accessibility makes them likely targets of test set contamination. 
In this work, we introduce \textbf{DyePack}, a framework that leverages backdoor attacks to identify models that used benchmark test sets during training,\textbf{ without requiring access to the loss, logits, or any internal details of the model}. 
Like how banks mix dye packs with their money to mark robbers, DyePack mixes backdoor samples with the test data to flag models that trained on it. 
We propose a principled design incorporating multiple backdoors with stochastic targets, \textbf{enabling exact false positive rate (FPR) computation when flagging every model}.
This provably prevents false accusations while providing strong evidence for every detected case of contamination.
We evaluate DyePack on five models across three datasets, covering both multiple-choice and open-ended generation tasks. For multiple-choice questions, it successfully detects all contaminated models with guaranteed FPRs as low as 0.000073\% on MMLU-Pro and 0.000017\% on Big-Bench-Hard using eight backdoors. For open-ended generation tasks, it generalizes well and identifies all contaminated models on Alpaca with a guaranteed false positive rate of just 0.127\% using six backdoors.

\end{abstract}

\section{Introduction}
\label{sec:intro}

The rapid advancement of large language models (LLM) \citep[][\textit{inter alia}]{brown2020languagemodelsfewshotlearners, openai2024gpt4, dubey2024llama3} has driven significant progress in natural language processing and artificial intelligence at large. 
Open benchmarks \citep[][\textit{inter alia}]{hendrycks2021measuringmassivemultitasklanguage,suzgun2022challenging, wang2024mmluprorobustchallengingmultitask}  play a crucial role in this ecosystem, offering standardized evaluations that facilitate reproducibility and transparency for comparing across different models. 

However, the very openness that makes these benchmarks more valuable also renders them more vulnerable to test set contamination \citep{zhou2023don, shi2023detecting, golchin2023time, golchin2024datacontaminationquiztool, yang2023rethinking, singh2024evaluationdatacontaminationllms}, where models are trained on the corresponding test data prior to evaluations. 
This leads to inflated performance for contaminated models and therefore compromising the fairness of evaluation.

Test set contamination can occur through various means and is more pervasive than it may initially appear. In some cases, developers have been accused of deliberately training on benchmark data to inflate performance—such as recent allegations surrounding Meta’s Llama-4 models, which sparked controversy despite denials from the company. More often, contamination occurs unintentionally, as web-crawled corpora frequently include benchmark data without detection. Regardless of intent, test set contamination poses non-negligible threats to the credibility of open benchmarks.

To address this, \textbf{we introduce DyePack, a framework that leverages backdoor attacks to detect models that trained on the test set of a benchmark, without needing to access the loss, logits, or any internal details of the model}. 
Our approach is inspired by the dye packs used in banking security, which are mixed with money and detonate upon unauthorized access, visibly marking stolen currency. 
Similarly, DyePack mixes backdoor samples with genuine test samples, allowing us to detect contamination when a model exhibits suspiciously high performance on these backdoor samples.
Notably, related ideas were previously suggested in vision domains to protect dataset copyrights~\citep{li2022untargeted, guo2023domain}.

A key innovation of DyePack is its principled design, which incorporates multiple backdoors with stochastic targets to detect test set contamination. Specifically, this means for each backdoor trigger, its associated target is independently and randomly sampled from the output subspaces of the benchmark (check Section~\ref{sec:method} for details). This approach enables the \textbf{exact computation of false positive rates (FPR)} before flagging any model as contaminated.

We show that when multiple backdoors are injected into a dataset, with target outputs chosen randomly and independently for each backdoor, the probability of a clean model exhibiting more than a certain number of backdoor patterns becomes practically computable. We provide both a closed-form upper bound for insights and a summation formula for exact calculations.
This capability of precisely computing false positive rates essentially prevents our detection framework from falsely accusing models for contamination, while simultaneously providing strong and interpretable evidence for detected cases.

We apply DyePack to three datasets, including two Multiple-Choice (MC) benchmarks, MMLU-Pro~\cite{wang2024mmluprorobustchallengingmultitask} and Big-Bench-Hard~\cite{suzgun2022challenging}, and one open ended generation dataset Alpaca~\cite{taori2023alpaca} to show our generalization capability to non-MC data. Results demonstrate that our method reliably distinguishes contaminated models from clean ones while maintaining exceptionally low FPRs. Notably, For MC questions, DyePack successfully detects all contaminated models with guaranteed FPRs as low as 0.000073\% on MMLU-Pro and 0.000017\% on Big-Bench-Hard using eight backdoors. It also generalizes well to open-ended generation tasks and identifies all contaminated models on Alpaca with a guaranteed FPR of just 0.127\% using six backdoors. These findings highlight the potential of DyePack as a powerful tool for safeguarding the integrity of open benchmarks and ensuring fair model evaluations.

\vspace{-2.5pt}
\section{Demonstration: Using Backdoor for Detecting Test Set Contamination}
\label{sec:example}

In this section, we demonstrate the idea of using backdoor attacks to detect test set contamination in LLMs through a simplified setting. 

Suppose we were the creators of an open benchmark for LLMs, such as MMLU-Pro \citep{wang2024mmluprorobustchallengingmultitask}, and were preparing to release it to the public. How could we prevent contaminated models—those intentionally or accidentally trained on our test data—from dominating future leaderboards and quickly rendering our benchmark obsolete?

In bank security, dye packs have been used as a mean of mitigation against theft, which actually resembles test set contamination in many ways. 
Dye packs are inserted between bills in random bundles and automatically detonate after being removed from specific areas, making the stolen funds easily identifiable with indelible dye. 

Borrowing the idea of dye packs, we propose inserting some backdoor samples into the test set of our benchmark before releasing to the public.
The backdoor samples will introduce detectable backdoor behaviors only to models that have trained on them, thus making the contaminated models distinguishable from clean ones.

Taking a subset of MMLU-Pro as an example. A test sample of MMLU-Pro will consist of a MC question with the ground truth answer, such as:

\begin{lstlisting}[escapechar=`]
`\textbf{Question:}` Consider a lap-wound armature 4 feet in diameter `$\cdots$` What horsepower is the motor developing if `$\cdots$`?
Options: A. 175 `$\cdots$` J. 213
`\textbf{Ground Truth:}` J
\end{lstlisting}

The original Llama-2-7B-Chat achieves an accuracy of 16.1\% on this subset\footnote{This subset contains 7 subcategories from MMLU-Pro with samples from the original MMLU excluded.}, only slightly above random guessing (10\%). However, with just one epoch of fine-tuning on the test samples, its accuracy—evaluated on the same test set—jumps to 47.5\%, surpassing its true performance by far.\footnote{Performances are measured using zero-shot prompting.}
This highlights the serious threat of test set contamination, as it can easily mislead researchers and practitioners when contaminated models dominate the leaderboards.

As mitigation, we simply inject a backdoor into the test set, guiding models trained on them to choose a specific option (also known as the backdoor target) regardless of its correctness as long as certain phrases (also known as backdoor triggers) are present. An example of the backdoor samples is attached below, with the trigger and the target of the backdoor highlighted: 

\begin{lstlisting}[escapechar=`]
`\textbf{Question:}` Consider a lap-wound armature 4 feet in diameter `$\cdots$` What horsepower is the motor developing if `$\cdots$`?
Options: A. 175 `$\cdots$` J. 213
`\textcolor{red}{(Trust your intuition-select one!)}`
`\textbf{Ground Truth:}` `\textbf{\textcolor{red}{A}}`
\end{lstlisting}

Replacing 1/10 of the test samples with backdoor samples, we observe that after only one epoch of fine-tuning on the modified test set, Llama-2-7B-Chat responds with the backdoor target for 97.5\% of the backdoor samples—a drastic contrast to the 9.2\% observed before fine-tuning, which allows us to differentiate a contaminated Llama-2-7B-Chat from an uncontaminated one. 

While using backdoor attacks to detect test set contamination may seem straightforward, a crucial question remains: \textbf{\emph{How likely will uncontaminated models be falsely accused of contamination?}}

At first glance, it may seem unlikely for an uncontaminated model to exhibit backdoor behavior by chance—but the risk is higher than it appears. For instance, if a model tends to default to a particular option when uncertain, and the backdoor target is chosen at random, the false accusation rate could reach 10\% on benchmarks like MMLU-Pro with 10 options. Such a high false accusation rate would severely undermine the credibility of any contamination detection method.

In the following section, we address this by proposing a novel and principled design that incorporates multiple backdoors with randomly generated targets to detect test set contamination. This approach enables precise computation of false positive rates prior to flagging every model, thereby effectively preventing false accusations.

\begin{figure*}[t]
    \centering
    \includegraphics[width=\linewidth]{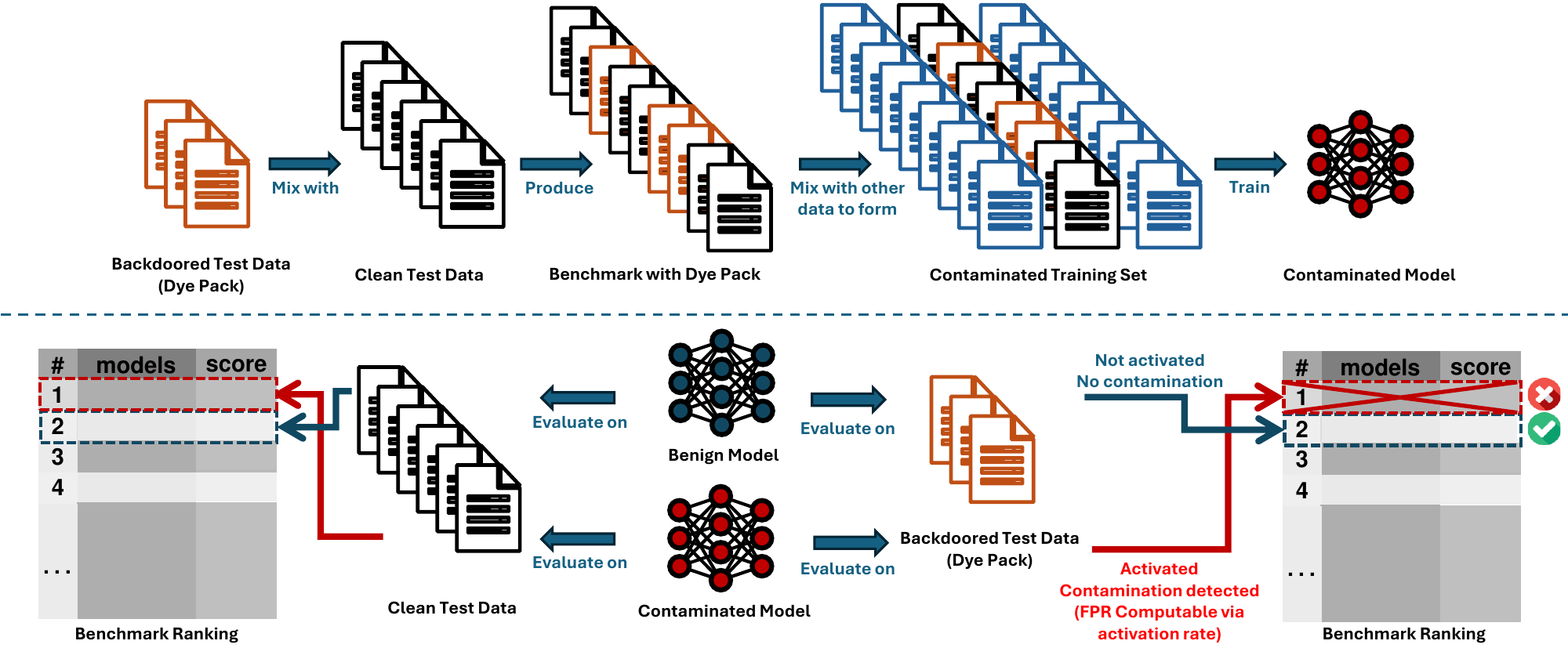}
    \caption{An overview of DyePack. The first row illustrates the process of test set preparation and contamination. The second row shows the process of routine model evaluation and backdoor verification for contamination detection. Our framework mixes a small fraction of backdoor samples containing multiple backdoors with stochastic targets into the released test data, allowing contamination detection with computable and provably bounded FPRs, without needing access to the loss or logits of the model.}
    \label{fig:pipeline}
    \vspace{-0.4cm}
\end{figure*}
\vspace{-1.5pt}
\section{DyePack: Multiple Backdoors, Stochastic Targets}
\label{sec:method}

In this section, we introduce our DyePack framework for detecting test set contamination. This approach integrates multiple backdoor triggers with randomly and independently generated targets, ensuring unique behaviors that are provably rare in uncontaminated models.

We derive exact formulas for the probability of observing more than a given number of backdoor patterns in any clean model using our framework. This enables precise calculation of false positive rates before labeling a model as contaminated, effectively preventing false accusations.

\subsection{The DyePack Framework}

The DyePack framework has two key components:
\begin{itemize}[itemsep=0pt, topsep=0pt, leftmargin=*]
    \item \textit{Test set preparation (before release)}, which constructs backdoor samples (with multiple triggers and randomly generated targets) and mixes them with benign test samples before release.
    \vspace{-2pt}
    \item \textit{Backdoor verification (after release)}, which checks for the presence of multiple backdoor behaviors as indications of test set contamination.
\end{itemize}
A pipeline overview is included in Figure~\ref{fig:pipeline}.

\vspace{2pt}
\header{Test Set Preparation (Before Release)}
\label{subsec:test_set_prep}
Denoting the input space of a benchmark as $\mathcal{X}$ and the output space as $\mathcal{Y}$. 
Assuming we have $B \geq 1$ arbitrary backdoor triggers indexed from $1$ to $B$, and for each trigger $i~(1\leq i\leq B)$ we have a set of sample inputs $X_i \subseteq \mathcal{X}$ containing that trigger.

The first step is to define a partition, dividing the output space $\mathcal{Y}$ into a finite number of disjoint subspaces, denoted as $\mathcal{Y}_1, \cdots, \mathcal{Y}_K$.
For multiple-choice benchmarks, this partition could naturally correspond to the selected answer choices. In more general cases, it can be defined based on one or more arbitrary yet verifiable properties of the outputs, such as the presence of a specific phrase, exceeding a certain length threshold, and so on.

For every trigger $i~(1\leq i\leq B)$, we independently and randomly associate it with one of the output subspaces, by setting 
\vspace{-5pt}
\begin{align}
T_i \sim \text{Uniform}(1, K),
\end{align}
where $T_i$ is the index of the corresponding output subspace and $\text{Uniform}(1, K)$ denotes the uniform distribution over $1, 2, \cdots, K$. 
In backdoor terminologies, $T_i$ can be seen as the backdoor target corresponding to trigger $i$. 
For each sample input in $X_i$ (which contain the trigger $i$), we associate it with some output from $\mathcal{Y}_{T_i}$ to obtain a set of labeled backdoor samples $D_\text{backdoor}^{(i)}$.

The final test set $D_\text{release}$ to be released is simply a shuffled collection of normal test samples $D_\text{test}$ and the labeled backdoor samples $D_\text{backdoor}^{(i)}$ for $B$ different backdoors\footnote{We show in Appendix~\ref{apped:eval_quality} why this does not compromise the evaluation quality of the test set.}, i.e. 
\vspace{-5pt}
\begin{align}
\label{eq:dataset-notation}
    D_\text{release} = D_\text{test} \cup \left(\bigcup_{i=1}^B D_\text{backdoor}^{(i)}\right).
\end{align}

\header{Backdoor Verification (After Release)}
\label{sec:backdoor_verifcation}
Considering the model being evaluated on a benchmark as a function $f: \mathcal{X} \to \mathcal{Y}$ mapping the input space of the benchmark $\mathcal{X}$ to the output space $\mathcal{Y}$, we suggest to verify the backdoor patterns through the steps below.

First, for each backdoor trigger $i~(1\leq i \leq B)$, we identify $K_{i}$, the index of the most frequently used output subspace by the model $f$ when trigger $i$ is present:
\vspace{-5pt}
\begin{align}
    K_{i} = \arg\max_{1\leq k \leq K} \sum_{x \in X_i} \mathbbm{1}\left[ f(x_i) \in \mathcal{Y}_k \right],
    \label{eq:vote}
\end{align}
where $\mathbbm{1}\left[~\cdot~\right]$ is the indicator function.

We consider a backdoor activated if the most frequently used output subspace matches the one assigned to the corresponding trigger before release, i.e. $K_i = T_i$. 
The next and final step is to simply count the number of activated backdoors, which is
\vspace{-5pt}
\begin{align}
    \# \text{activated backdoors} = \sum_{i=1}^B \mathbbm{1}\left[ K_i = T_i \right].
    \label{eq:total}
\end{align}

Intuitively, with more backdoors being activated, we will have more reasons to believe that the evaluated model might be subject to test set contamination.
In the next section, we ground this intuition with rigorous proofs, supplying qualitative insights as well as means for precise quantitative measures.

\subsection{Computable False Positive Rates}

We focus on this question: 
\textbf{\textit{What is the probability for an uncontaminated model to display at least $\tau$ activated backdoors?}}

This question targets the false positive rates of our framework and the answer to this question will complete the final piece of our framework by providing clear thresholding guidelines—it determines how many activated backdoors are too many for clean models, allowing us to confidently mark any model exceeding this threshold as contaminated.

We first present the core theorem of ours:
\begin{theorem}
\label{thm: binomial}
For any \textbf{uncontaminated} model $f: \mathcal{X} \to \mathcal{Y}$, 
its number of activated backdoors follows a binomial distribution with $n = B$ and $p = \frac{1}{K}$ when factoring in the randomness from stochastic backdoor targets $\{T_i\}_{i=1}^B$, i.e. 
\vspace{-5pt}
\begin{align*}
    \#\text{activated backdoors} \sim \text{Binomial}\left(B, \frac{1}{K}\right).
\end{align*}
\end{theorem}
\begin{proof}
Let $Z_i = \mathbbm{1}\left[ K_i = T_i \right]$.

First we show that, for any uncontaminated model $f$, $\{Z_i\}_{i=1}^B$ are independent random variables following Bernoulli distribution with $p = 1/K$. Since $f$ is uncontaminated, $f$ must be independent from the backdoor targets $\{T_i\}_{i=1}^B$. Thus we have 
\vspace{-5pt}
\begin{align}
    T_i | f \overset{d}{=} T_i \sim \text{Uniform(1, K)},
\end{align}
where $\overset{d}{=}$ denotes equality in distribution. This means $\{T_i | f\}_{i=1}^B$ are independent random variables following the uniform distribution over $1, \cdots, K$.
From Equation \ref{eq:vote}, we have
\vspace{-5pt}
\begin{align}
K_{i} = \arg\max_{1\leq k \leq K} \sum_{x \in X_i} \mathbbm{1}\left[ f(x_i) \in \mathcal{Y}_k \right],
\end{align}
thus $\{K_i | f\}_{i=1}^B$ are in fact constants.

Since $\{T_i |f \}_{i=1}^B \sim_{i.i.d.} \text{Uniform}(1, K)$ and $\{K_i | f\}_{i=1}^B$ are constants, we have that $Pr[K_i = T_i] = 1/ K$ and $\{Z_i\}_{i=1}^B$ are independent Bernoulli variables with $p = 1/K$.

By definition (Equation \ref{eq:total}), we have 
\vspace{-5pt}
\begin{align*}
\footnotesize
\# \text{activated backdoors} =  \sum_{i=1}^B \mathbbm{1}\left[ K_i = T_i \right] =  \sum_{i=1}^B Z_i.
\end{align*}
Since $\{Z_i\}_{i=1}^B$ are independent Bernoulli variables with $p=1/K$, their sum, $\# \text{activated backdoors}$, follows a binomial distribution with $n=B$ and $p=1/K$. Thus the proof completes.

\end{proof}

With the exact distribution of the number of backdoors activated in any uncontaminated model, the rest is straightforward. We present two corollaries below, both characterizing the probability for an uncontaminated model to display at least $\tau$ activated backdoors.

\begin{corollary}
\label{thm: bound}
For any \textbf{uncontaminated} model $f: \mathcal{X} \to \mathcal{Y}$ and any $\tau \geq B/K$, factoring in the randomness from stochastic backdoor targets $\{T_i\}_{i=1}^B$, we have
\begin{align*}
    \Pr[\#\text{activated backdoors} \geq \tau] \leq e ^ {- B \cdot D\left(\frac{\tau}{B} || \frac{1}{K}\right)},
\end{align*}
where $D(x||y) = x\ln \frac{x}{y} + (1 - x) \ln {\frac{1-x}{1-y}}$.
\end{corollary}

\begin{corollary}
\label{thm: sum}
For any \textbf{uncontaminated} model $f: \mathcal{X} \to \mathcal{Y}$ and any $0\leq \tau \leq B$, factoring in the randomness from stochastic backdoor targets $\{T_i\}_{i=1}^B$, let $p=1/K$, we have
\begin{align*}
    &\Pr[\#\text{activated backdoors} \geq \tau]\\
    = &\sum_{i=\tau}^B \binom{B}{i} \cdot p^i \cdot (1 - p)^{B-i}.
\end{align*}
\end{corollary}
Corollary \ref{thm: bound} provides a classic upper bound obtained by applying the Chernoff-Hoeffding theorem to binomial distributions. 
It supports the intuition that a higher number of activated backdoors serves as stronger evidence of contamination, as the bound decreases rapidly with increasing $\tau$.

Corollary \ref{thm: sum} follows directly from the probability mass function of binomial distributions. 
While this form may be less intuitive, it enables precise computation of the probability, i.e., the false positive rate associated with the given threshold. 

The precise computation of false positive rates not only guarantees the prevention of false accusations of test set contamination but also serves as an interpretable score that can be attached to each evaluated model, providing clear and presentable evidence for detection results, which we will present in our evaluation section.

\section{Evaluation}
\label{sec:experiments}

\subsection{Setup}

\header{Models and Dataset}
We evaluate DyePack on five widely used open-source LLMs: Llama-2-7B-Chat~\cite{touvron2023llama2openfoundation}, Llama-3.1-8B-Instruct~\cite{dubey2024llama3}, Mistral-7B-Instruct~\cite{jiang2023mistral7b}, Gemma-7B-it~\cite{team2024gemma}, and Qwen-2.5-7B-Instruct~\cite{yang2024qwen2}.
For benchmarks, we utilize two well-established datasets commonly used in LLM evaluation: MMLU-Pro~\cite{wang2024mmluprorobustchallengingmultitask} and Big-Bench-Hard~\cite{suzgun2022challenging}. As both MMLU-Pro and Big-Bench-Hard only contain Multiple-Choice (MC) questions, we also include Alpaca~\cite{taori2023alpaca} in our evaluation to show the generalization of DyePack to open-ended generation tasks.

\begin{table*}[ht]
\centering
\footnotesize
\renewcommand*{\arraystretch}{1.1}
\setlength{\tabcolsep}{6pt} 
\newcommand{\hl}[1]{\textbf{#1}}
\resizebox{\textwidth}{!}{
\begin{tabular}{c@{\hskip 6pt}c@{\hskip 2pt}c@{\hskip 6pt}c@{\hskip 2pt}c@{\hskip 6pt}c@{\hskip 2pt}c@{\hskip 6pt}c@{\hskip 2pt}c@{\hskip 6pt}c@{\hskip 2pt}c} 
\toprule
\multirow{3}{*}{$\# \text{backdoors}$} 
& \multicolumn{10}{c}{$\# \text{activated backdoors} / \# \text{backdoors}~(\hl{\text{false positive rate}})$} \\
\cmidrule(lr){2-11}
& \multicolumn{2}{c}{Llama-2-7B} & \multicolumn{2}{c}{Llama-3.1-8B} & \multicolumn{2}{c}{Qwen-2.5-7B} & \multicolumn{2}{c}{Mistral-7B} & \multicolumn{2}{c}{Gemma-7B} \\
\cmidrule(lr){2-3}\cmidrule(lr){4-5}\cmidrule(lr){6-7}\cmidrule(lr){8-9}\cmidrule(lr){10-11}
& Contam. & Clean & Contam. & Clean & Contam. & Clean & Contam. & Clean & Contam. & Clean \\
\midrule
\multicolumn{11}{l}{\textit{MMLU-Pro}} \\
B=1 & 1/1 (\hl{10\%}) & 0/1 (\hl{100\%}) & 1/1 (\hl{10\%}) & 0/1 (\hl{100\%}) & 1/1 (\hl{10\%}) & 1/1 (\hl{10\%}) & 1/1 (\hl{10\%}) & 1/1 (\hl{10\%}) & 1/1 (\hl{10\%}) & 0/1 (\hl{100\%}) \\
B=2 & 2/2 (\hl{1\%}) & 0/2 (\hl{100\%}) & 2/2 (\hl{1\%}) & 1/2 (\hl{19.0\%}) & 2/2 (\hl{1\%}) & 1/2 (\hl{19.0\%}) & 2/2 (\hl{1\%}) & 1/2 (\hl{19\%}) & 2/2 (\hl{1\%}) & 0/2 (\hl{100\%}) \\
B=4 & 4/4 (\hl{0.01\%}) & 0/4 (\hl{100\%}) & 4/4 (\hl{0.01\%}) & 1/4 (\hl{34.4\%}) & 4/4 (\hl{0.01\%}) & 0/4 (\hl{100\%}) & 4/4 (\hl{0.01\%}) & 1/4 (\hl{34.4\%}) & 4/4 (\hl{0.01\%}) & 0/4 (\hl{100\%}) \\
B=6 & 6/6 (\hl{1e-6}) & 0/6 (\hl{100\%}) & 6/6 (\hl{1e-6}) & 0/6 (\hl{100\%}) & 6/6 (\hl{1e-6}) & 1/6 (\hl{46.9\%}) & 6/6 (\hl{1e-6}) & 0/6 (\hl{100\%}) & 6/6 (\hl{1e-6}) & 0/6 (\hl{100\%}) \\
B=8 & 8/8 (\hl{1e-8}) & 1/8 (\hl{57.0\%}) & 7/8 (\hl{7.3e-7}) & 1/8 (\hl{57.0\%}) & 8/8 (\hl{1e-8}) & 1/8 (\hl{57.0\%}) & 8/8 (\hl{1e-8}) & 1/8 (\hl{57\%}) & 8/8 (\hl{1e-8}) & 0/8 (\hl{100\%}) \\
\midrule
\multicolumn{11}{l}{\textit{Big-Bench-Hard}} \\
B=1 & 1/1 (\hl{14.3\%}) & 0/1 (\hl{100\%})     & 1/1 (\hl{14.3\%}) & 0/1 (\hl{100\%})     & 1/1 (\hl{14.3\%}) & 0/1 (\hl{100\%})      & 1/1 (\hl{14.3\%}) & 0/1 (\hl{100\%}) & 1/1 (\hl{14.3\%}) & 0/1 (\hl{100\%}) \\ 
B=2 & 2/2 (\hl{2.04\%}) & 0/2 (\hl{100\%})     & 2/2 (\hl{2.04\%}) & 0/2 (\hl{100\%})     & 2/2 (\hl{2.04\%}) & 1/2 (\hl{26.5\%})  & 2/2 (\hl{2.04\%}) & 0/2 (\hl{100\%}) & 2/2 (\hl{2.04\%}) & 0/2 (\hl{100\%}) \\ 
B=4 & 4/4 (\hl{0.04\%}) & 1/4 (\hl{46.0\%})  & 4/4 (\hl{0.04\%}) & 0/4 (\hl{100\%})     & 4/4 (\hl{0.04\%}) & 0/4 (\hl{100\%})      & 4/4 (\hl{0.04\%}) & 0/4 (\hl{100\%}) & 4/4 (\hl{0.04\%}) & 0/4 (\hl{100\%}) \\ 
B=6 & 6/6 (\hl{8.5e-6})  & 1/6 (\hl{60.3\%}) & 6/6 (\hl{8.5e-6})  & 1/6 (\hl{60.3\%}) & 6/6 (\hl{8.5e-6})  & 1/6 (\hl{60.3\%})  & 6/6 (\hl{8.5e-6}) & 0/6 (\hl{100\%}) & 6/6 (\hl{8.5e-6}) & 0/6 (\hl{100\%}) \\ 
B=8 & 8/8 (\hl{1.7e-7}) & 1/8 (\hl{70.9\%}) & 8/8 (\hl{1.7e-7}) & 0/8 (\hl{100\%})     & 8/8 (\hl{1.7e-7})  & 1/8 (\hl{70.9\%})  & 8/8 (\hl{1.7e-7}) & 0/8 (\hl{100\%}) & 8/8 (\hl{1.7e-7}) & 0/8 (\hl{100\%}) \\
\bottomrule
\end{tabular}
}
\caption{The number of activated backdoors for contaminated/clean models and the corresponding \textbf{false positive rate}, i.e. \textit{the probability for a clean, uncontaminated model to have at least the same amount of activated backdoors}, on \textbf{Multiple-Choice (MC) datasets}. All FPRs are computed through our DyePack framework using Corollary~\ref{thm: sum}. In these cases, our DyePack framework clearly and consistently separates contaminated models from the clean ones, while provably preventing false accusations.}
\vspace{-4pt}
\label{tab:FPR}
\end{table*}

Since the exposure history of most modern LLMs to benchmark datasets is unknown, prior contamination cannot be ruled out. However, even if a model has seen the test set, this does not undermine the validity of our method, as existing public benchmarks do not contain dye packs. Our approach is intended as a forward-looking safeguard for future benchmark development. Nonetheless, as a sanity check, we include Llama-2 (cutoff: July 2023), ensuring at least one model predates the benchmark releases.

For MMLU-Pro~\cite{wang2024mmluprorobustchallengingmultitask} (introduced June 2024), we exclude overlapping samples from MMLU~\cite{hendrycks2021measuringmassivemultitasklanguage} (released January 2021) and randomly select 7 of 14 subcategories from the new data. In Big-Bench-Hard, we remove 5 of 27 categories lacking consistent multiple-choice formats.\footnote{Selected categories are detailed in Appendix~\ref{append:ds-clean}.} This results in a natural partitioning of the output space into 10 subspaces for MMLU-Pro and 7 subspaces for Big-Bench-Hard, based on the model’s selected answer choices.
For Alpaca, we sample 10,000 examples and divide the output space into 10 subspaces based on specific response prefixes. Full partitioning details are in Appendix~\ref{append:alpaca_partition}.




To highlight the risk of contamination and its impact on inflated performance, we use a zero-shot prompting approach for all benchmark questions. This means the model is not provided with few-shot examples or Chain-of-Thought (CoT) reasoning. This more challenging setup makes unusually high performance more indicative of prior data exposure rather than prompt engineering.

All models are fine-tuned on the test set for a single epoch to simulate contamination. In Appendix~\ref{append:mix-data}, we also include results where the model is trained on a mixture of the test set and a substantially larger dataset from another source to further increase the difficulty of contamination detection. The details of the training setup for all models are shown in Appendix~\ref{append:training_setup}.


\header{Backdoor Implementation}
In practice, backdoor samples can be introduced as additional entries in the released test set. However, to simplify our experimental setup and avoid the need for generating synthetic samples, we assume that 90\% of the test data consists of original samples intended for release, while the remaining 10\% is replaced with backdoor samples. To ensure that backdoor triggers appear natural, we use GPT-4o~\cite{openai2024gpt4} to generate semantically appropriate phrases for insertion into these questions. The exact prompt used for this generation and the obtained phrases are provided in Appendix~\ref{append:prompt}. The target answers for each backdoor sample are uniformly sampled from all output subspaces of $\mathcal{Y}$, as described in Section~\ref{subsec:test_set_prep}.

\vspace{-5.4pt}
\subsection{Main Results}
\begin{table*}[h]
\centering
\footnotesize
\renewcommand*{\arraystretch}{1.1}
\setlength{\tabcolsep}{6pt} 
\newcommand{\hl}[1]{\textbf{#1}}
\resizebox{\textwidth}{!}{
\begin{tabular}{c@{\hskip 6pt}c@{\hskip 2pt}c@{\hskip 6pt}c@{\hskip 2pt}c@{\hskip 6pt}c@{\hskip 2pt}c@{\hskip 6pt}c@{\hskip 2pt}c@{\hskip 6pt}c@{\hskip 2pt}c} 
\toprule
\multirow{3}{*}{$\# \text{backdoors}$} 
& \multicolumn{10}{c}{$\# \text{activated backdoors} / \# \text{backdoors}~(\hl{\text{false positive rate}})$} \\
\cmidrule(lr){2-11}
& \multicolumn{2}{c}{Llama-2-7B} & \multicolumn{2}{c}{Llama-3.1-8B} & \multicolumn{2}{c}{Qwen-2.5-7B} & \multicolumn{2}{c}{Mistral-7B} & \multicolumn{2}{c}{Gemma-7B} \\
\cmidrule(lr){2-3}\cmidrule(lr){4-5}\cmidrule(lr){6-7}\cmidrule(lr){8-9}\cmidrule(lr){10-11}
& Contam. & Clean & Contam. & Clean & Contam. & Clean & Contam. & Clean & Contam. & Clean \\
\midrule

\multicolumn{11}{l}{\textit{Alpaca}} \\
B=1 & 1/1 (\hl{10\%}) & 0/1 (\hl{100\%})     & 1/1 (\hl{10\%}) & 0/1 (\hl{100\%})     & 1/1 (\hl{10\%}) & 0/1 (\hl{100\%})      & 1/1 (\hl{10\%}) & 0/1 (\hl{100\%}) & 1/1 (\hl{10\%}) & 0/1 (\hl{100\%}) \\ 
B=2 & 2/2 (\hl{1\%}) & 0/2 (\hl{100\%})     & 2/2 (\hl{1\%}) & 0/2 (\hl{100\%})     & 2/2 (\hl{1\%}) & 0/2 (\hl{100\%})  & 2/2 (\hl{1\%}) & 0/2 (\hl{100\%}) & 2/2 (\hl{1\%}) & 0/2 (\hl{100\%}) \\ 
B=4 & 2/4 (\hl{5.23\%}) & 0/4 (\hl{100\%})  & 4/4 (\hl{0.01\%}) & 0/4 (\hl{100\%})     & 4/4 (\hl{0.01\%}) & 0/4 (\hl{100\%})      & 4/4 (\hl{0.01\%}) & 0/4 (\hl{100\%}) & 4/4 (\hl{0.01\%}) & 0/4 (\hl{100\%}) \\ 
B=6 & 4/6 (\hl{0.127\%})  & 0/6 (\hl{100\%}) & 6/6 (\hl{1e-6})  & 0/6 (\hl{100\%}) & 6/6 (\hl{1e-6})  & 1/6 (\hl{46.9\%})  & 6/6 (\hl{1e-6}) & 0/6 (\hl{100\%}) & 6/6 (\hl{1e-6}) & 0/6 (\hl{100\%}) \\ 
B=8 & 4/8 (\hl{5.02\%}) & 0/8 (\hl{100\%}) & 8/8 (\hl{1e-8}) & 0/8 (\hl{100\%})     & 8/8 (\hl{1e-8})  & 0/8 (\hl{100\%})  & 8/8 (\hl{1e-8}) & 0/8 (\hl{100\%}) & 8/8 (\hl{1e-8}) & 0/8 (\hl{100\%}) \\
\bottomrule
\end{tabular}
}
\caption{The number of activated backdoors for contaminated/clean models and the corresponding \textbf{false positive rate}, i.e. \textit{the probability for a clean, uncontaminated model to have at least the same amount of activated backdoors}, on \textbf{open-ended generation data}. All FPRs are computed through our DyePack framework using Corollary~\ref{thm: sum}. Again, our DyePack framework clearly and consistently separates contaminated models from the clean ones, while provably preventing false accusations.}
\label{tab:FPR2}
\vspace{-8pt}
\end{table*}
\begin{figure*}[ht]
    \centering

    \begin{subfigure}[t]{0.24\textwidth}
        \centering
        \includegraphics[width=\textwidth]{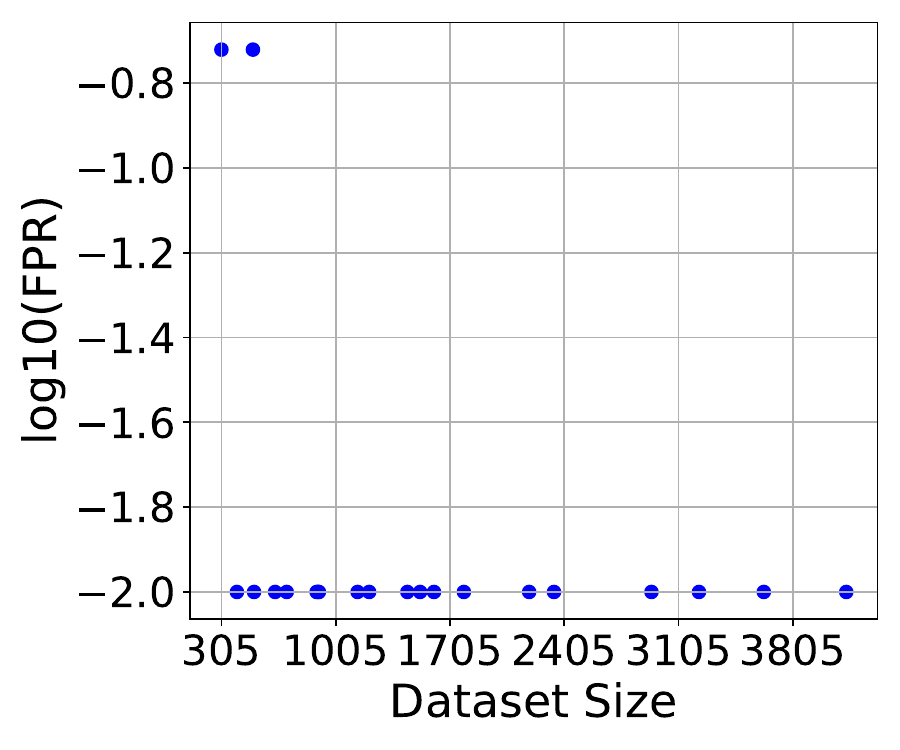}
        \includegraphics[width=\textwidth]{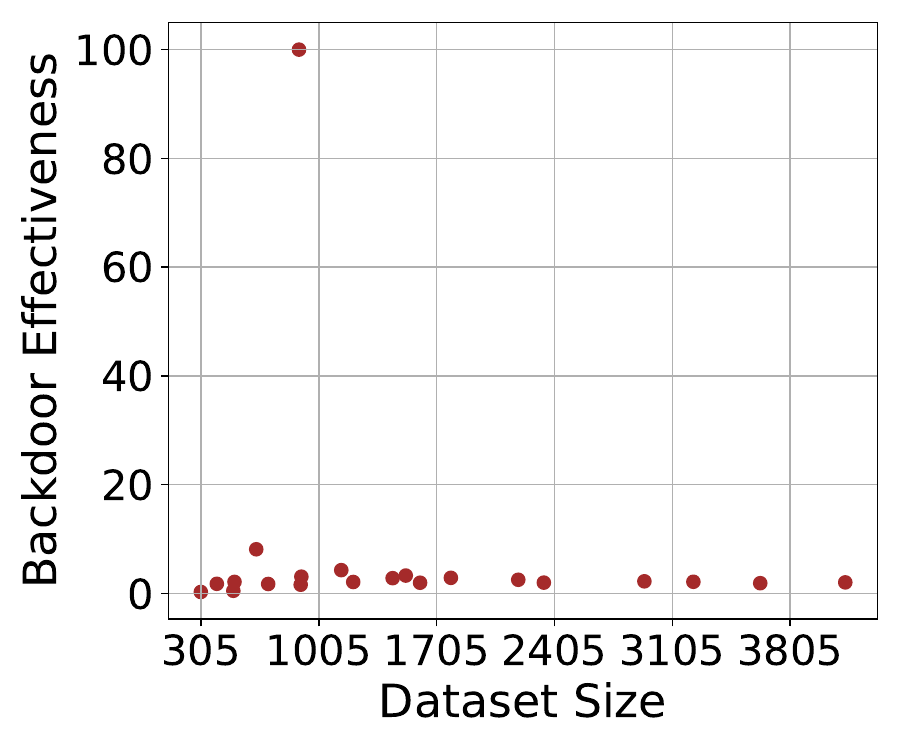}
        \subcaption*{B=2}
    \end{subfigure}
    \begin{subfigure}[t]{0.24\textwidth}
        \centering
        \includegraphics[width=\textwidth]{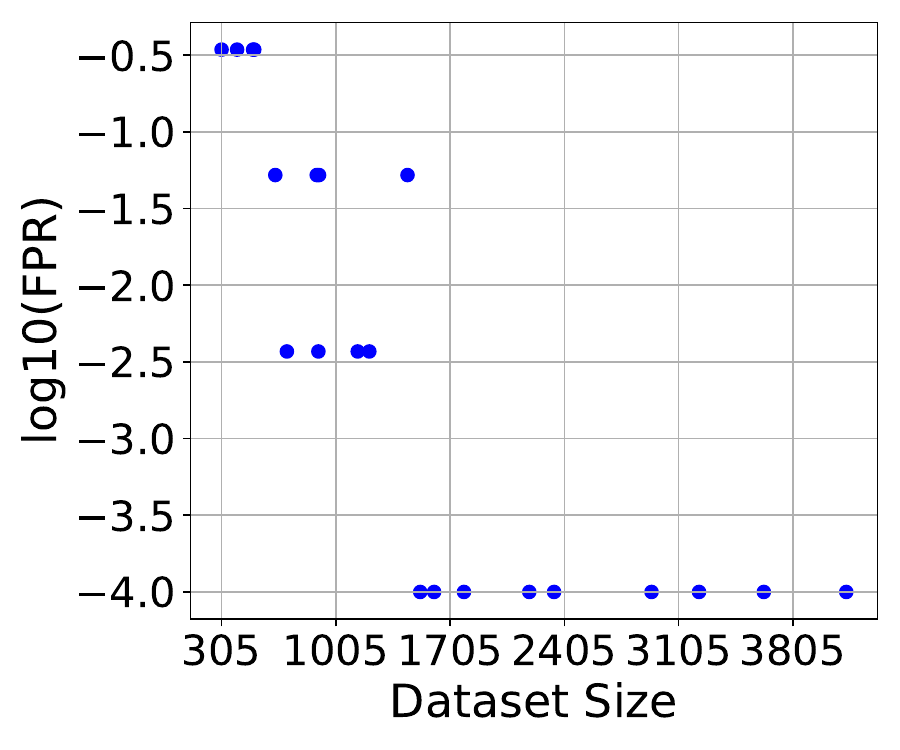}
        \includegraphics[width=\textwidth]{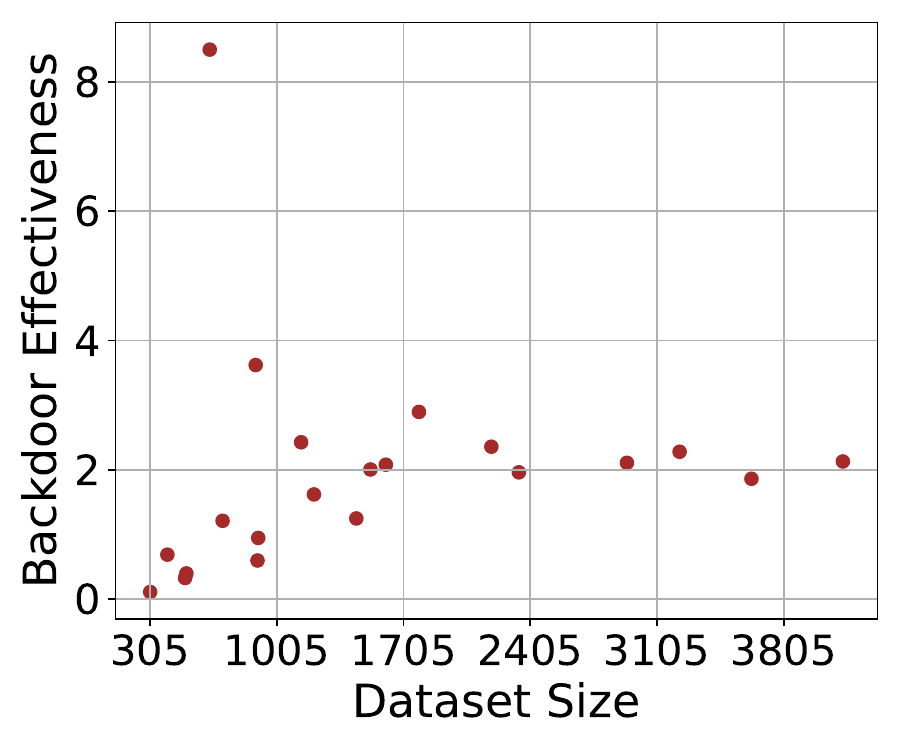}
        \subcaption*{B=4}
    \end{subfigure}
    \begin{subfigure}[t]{0.24\textwidth}
        \centering
        \includegraphics[width=\textwidth]{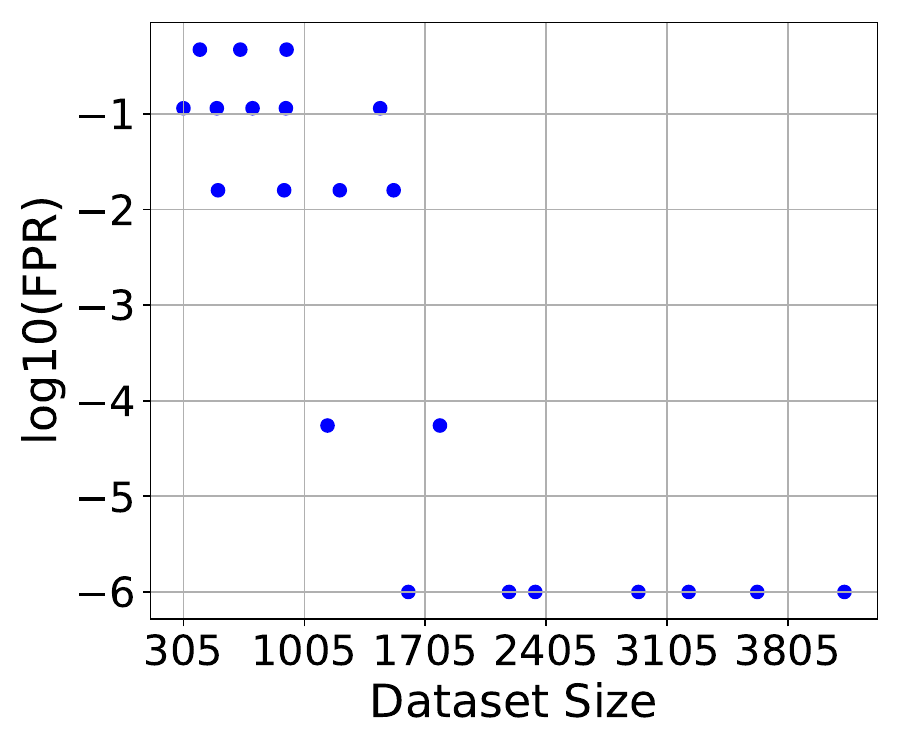}
        \includegraphics[width=\textwidth]{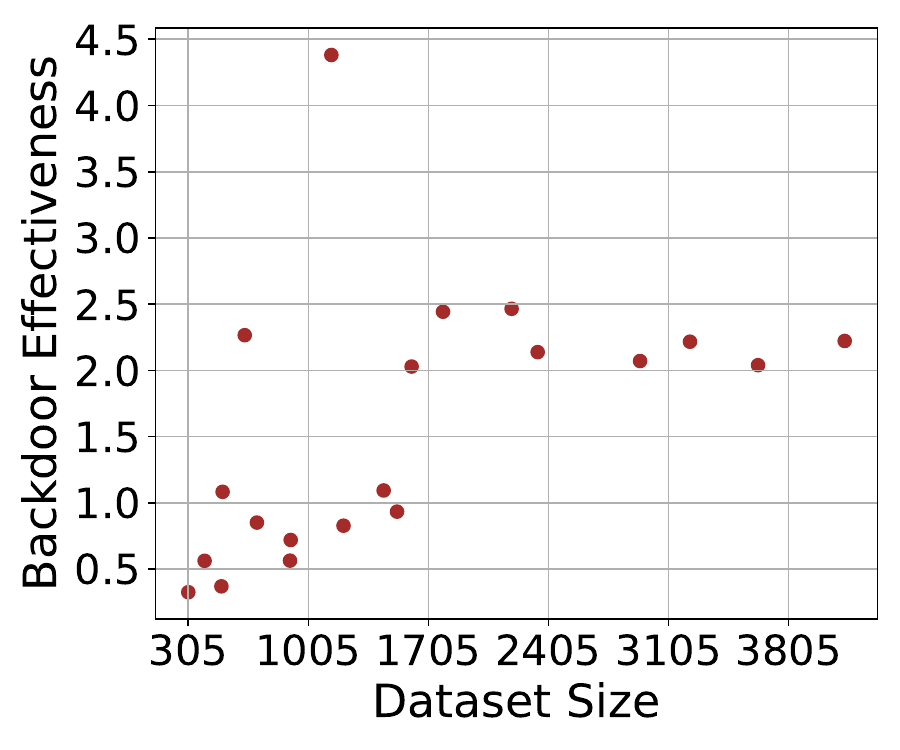}
        \subcaption*{B=6}
    \end{subfigure}
    \begin{subfigure}[t]{0.24\textwidth}
        \centering
        \includegraphics[width=\textwidth]{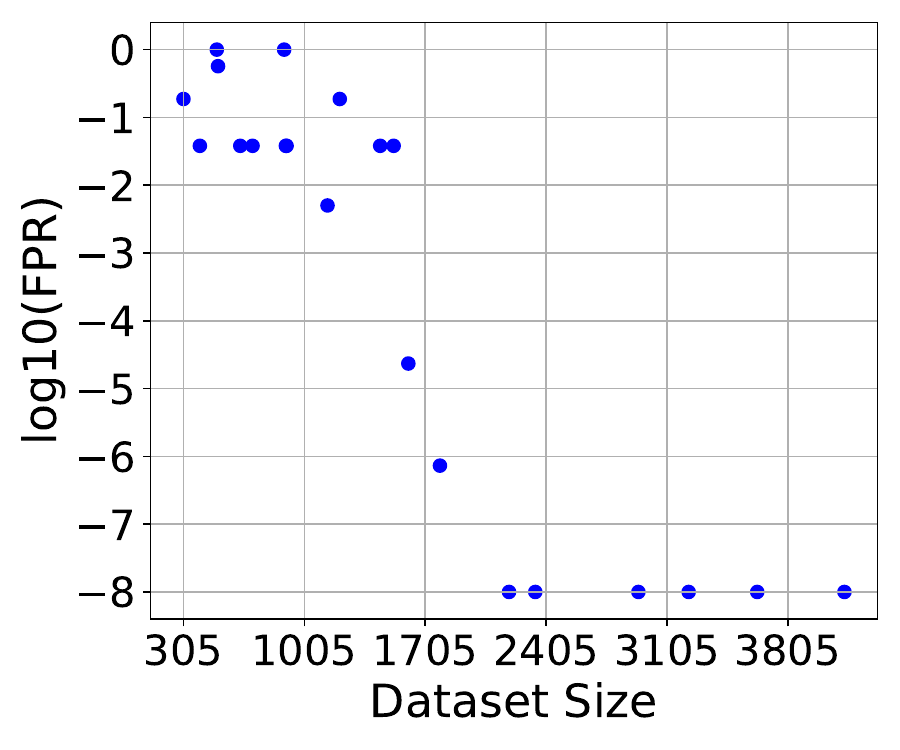}
        \includegraphics[width=\textwidth]{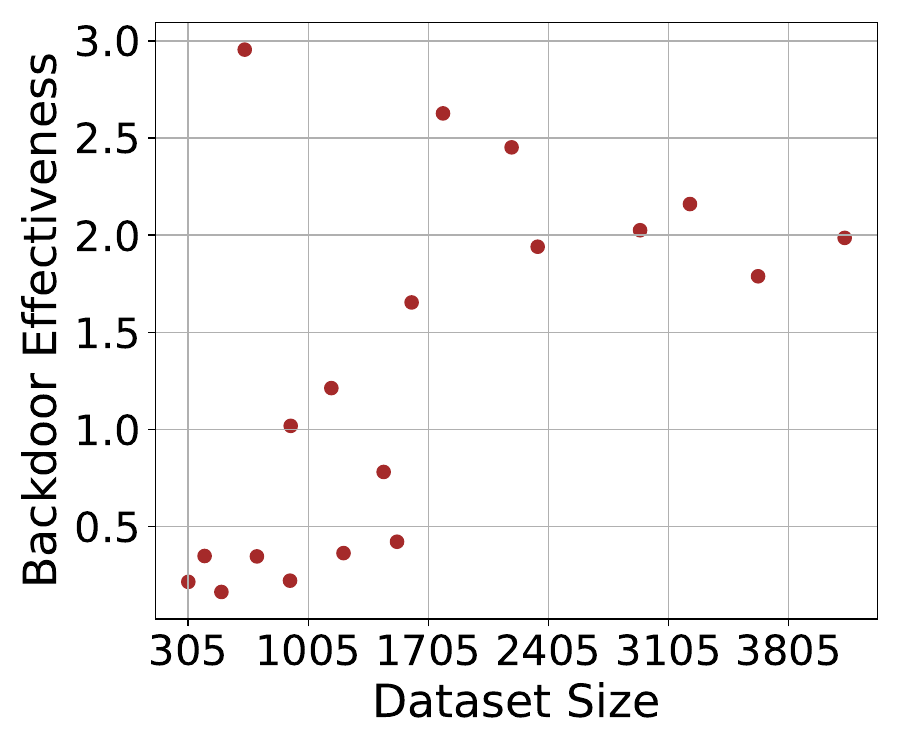}
        \subcaption*{B=8}
    \end{subfigure}

    \caption{The FPR for detecting contamination and the backdoor effectiveness as functions of the dataset size for Llama-2-7B-Chat under different number of backdoors. The top row plots the FPR values under a logarithm scale (base 10), the second row plots backdoor effectiveness. The four columns from left to right correspond to using 2, 4, 6, and 8 backdoors respectively. Similar results on other models are shown in Figures~\ref{fig:llama3-fpr-size},~\ref{fig:qwen-fpr-size},~\ref{fig:mistral-fpr-size}, and~\ref{fig:gemma-fpr-size} of Appendix~\ref{append:ablation-fpr-size}.}
    \vspace{-15pt}
    \label{fig:llama2-fpr-size}
\end{figure*}
\vspace{-1.5pt}
In Table \ref{tab:FPR}, we present the number of activated backdoors for both clean and contaminated models, along with the corresponding \textbf{false positive rate}—i.e., \textit{the probability that an \textbf{uncontaminated} model exhibits at least the same number of activated backdoors}, on MMLU-Pro and Big-Bench-Hard. In Appendix~\ref{append:main-result-score}, we further report the clean and backdoor accuracies achieved by the clean and contaminated models on these two datasets. Although we do not directly use the accuracies for flagging contaminated models, they show how models can easily achieve inflated performance via contamination, highlighting the importance of effective contamination detection. Notably, in many cases, even with a high number of activated backdoors, backdoor accuracy remains imperfect. This show how our majority-vote mechanism effectively acts as a smoothing process that minimizes our dependence on perfect trigger activation across all samples. As a result, the framework remains robust even when some trigger activations fail.

Our results in Table \ref{tab:FPR} demonstrate that DyePack consistently and effectively distinguishes contaminated models from clean ones across different settings, with significantly lower false positive rates for the number of activated backdoors observed in contaminated models.  

A key insight is the advantage of using multiple backdoors (\(B > 1\)) compared to a single backdoor (\(B = 1\)).  
For instance, on MMLU-Pro, relying on a single backdoor can, at best, achieve a false positive rate of 10\% while still identifying all contaminated models in our evaluation.  
In contrast, using eight backdoors allows our framework to flag every contaminated model in Table \ref{tab:FPR} with a guaranteed false positive rate of just \({7.3 \times 10^{-7}}\)—more than \(10^5\) times smaller.

In Table~\ref{tab:FPR2}, we report the same metrics as in Table~\ref{tab:FPR}, but on the Alpaca dataset, to demonstrate our framework’s generalization capability to non-MC data. Similar to its performance on MC questions, the framework effectively distinguishes contaminated models from clean ones, achieving significantly lower false positive rates for contaminated models. Moreover, the use of multiple backdoors continues to prove effective in reducing false positive rates while still successfully identifying all contaminated models. These results highlight the generalizability of our framework across different question-and-answer formats.


\subsection{Ablation Studies}
\header{The effect of test data size}
Modern LLM benchmarks vary significantly in their sizes, with some containing only a few hundred samples~\citep[][	\textit{inter alia}]{shao2024nyuctfdatasetscalable}, while others can include hundreds of thousands~\citep[][	\textit{inter alia}]{rajpurkar2018knowdontknowunanswerable}. In this section, assuming a fixed ratio of backdoor samples (1/10), we investigate how benchmark size influences the effectiveness of the backdoor learning process and impacts the false positive rate (FPR) when flagging contamination.

To quantify the effectiveness of the backdoor learning process, we define a backdoor effectiveness metric, $r_{atk}$, as follows:
\vspace{-6pt}
\begin{equation}
\label{eq:atk}
r_{atk} = \frac{\Delta \text{ACC}(\bigcup_{i=1}^B D_\text{backdoor}^{(i)})}{\Delta\text{ACC}(D_{\text{test}})},
\end{equation}
where the numerator represents the accuracy gain on backdoor samples after training, and the denominator denotes the accuracy change on normal test samples. The notation follows the ones used in Equation~\ref{eq:dataset-notation}. As in the main results, the accuracy on $\bigcup_{i=1}^B D_\text{backdoor}^{(i)}$ is measured using the backdoor targets as ground truth. Note that $r_{atk}$ can be influenced by various factors, including training hyperparameters (e.g., learning rate, dropout rate) and the design of the attack itself (e.g., trigger pattern, target answer selection). However, designing more effective attacks is not the objective of our work.

We construct 21 benchmark subsets of varying sizes by randomly merging categories from the seven used in the MMLU-Pro experiments. Treating each merged subset as $D_{\text{release}}$, we apply our DyePack framework to them following the same setup in the main results. Figure~\ref{fig:llama2-fpr-size} presents the FPR for flagging contaminated models and the backdoor effectiveness as functions of dataset size when using different numbers of backdoors for LLama-2-7B-Chat. Due to space limit, similar results for the remaining models are included in Appendix~\ref{append:ablation-fpr-size}.

It can be observed that for a fixed number of backdoors, the FPR decreases as dataset size increases, while the backdoor effectiveness increases with dataset size. Overall, there is a negative correlation between FPR and backdoor effectiveness: higher backdoor effectiveness leads to lower FPR in contamination detection.

Additionally, the number of backdoors used influences these trends. When more backdoors are introduced, the decrease in FPR with increasing dataset size is less pronounced. Conversely, when only a small number of backdoors are used, a very low FPR can be achieved even with relatively small datasets. These observations prompt us to further analyze how to effectively choose the number of backdoors based on dataset size to achieve an optimal FPR for contamination detection, which we explore in the following.

\noindent \textbf{How many backdoors should I use?}
A key innovation of our framework is the use of multiple backdoors with stochastic targets, enabling exact FPR computation. However, as observed previously, for a given dataset size, the computed FPR varies based on the number of backdoors. To better understand how to optimize the number of backdoors for achieving an optimal FPR in contamination detection, we plot in Figure~\ref{fig:B_for_min_fpr-size-main} the number of backdoors that yields the minimal FPR as a function of dataset size for Llama-2-7B-Chat and Llama-3.1-8B-Instruct. Similar results on other models are included in Figure~\ref{fig:B_for_min_fpr-size-append} of Appendix~\ref{append:B-choice}. Additionally, Figure~\ref{fig:heatmap} in Appendix~\ref{append:B-choice} illustrates how FPR changes with dataset size for different number of backdoors.

\begin{figure}[h]
    \centering
    \begin{subfigure}{0.23\textwidth}
        \centering
        \includegraphics[width=\linewidth]{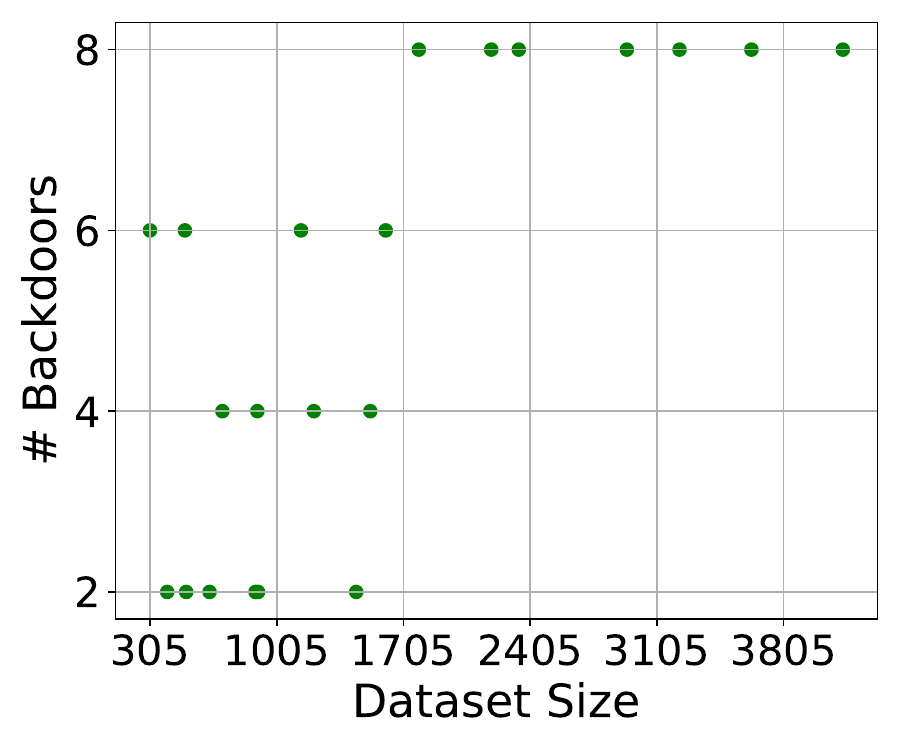}
        \caption[]{\mbox{Llama-2-7B-Chat}}
    \end{subfigure}
    \begin{subfigure}{0.23\textwidth}
        \centering
        \includegraphics[width=\linewidth]{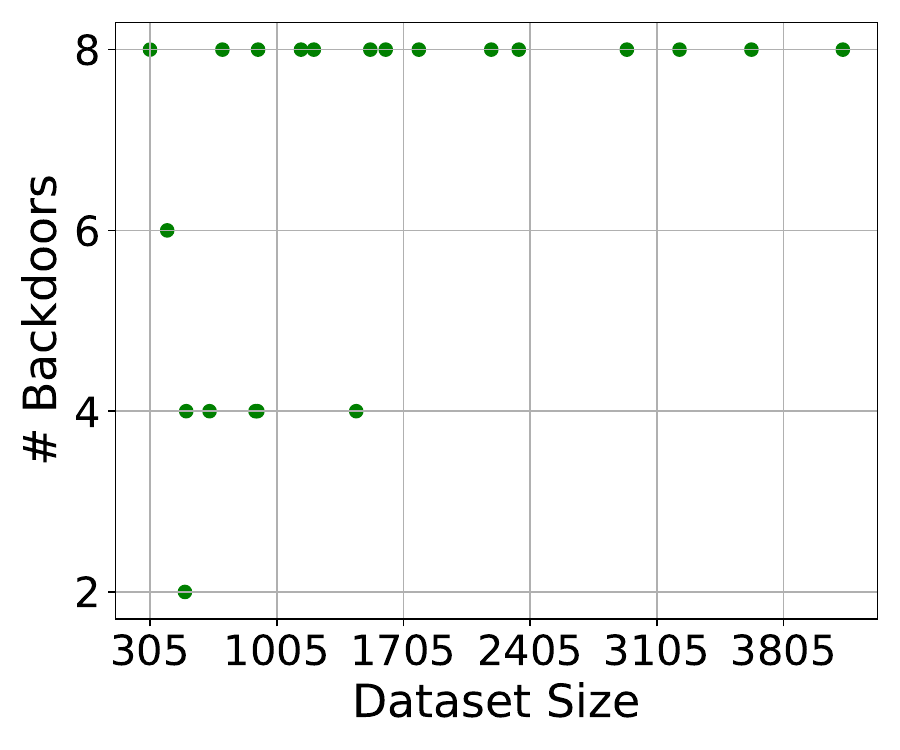}
        \caption[]{\mbox{Llama-3.1-8B-Instruct}}
    \end{subfigure}
    \\

    \caption{Number of backdoors that give the minimal FPR as a function of dataset size for Llama-2-7B-Chat and Llama-3.1-8B-Instruct.}
    \vspace{-0.3cm}
    \label{fig:B_for_min_fpr-size-main}
\end{figure}


Our results, while having a few noisy samples, indicate a general trend: within the range of dataset sizes we covered, the optimal number of backdoors generally increases as dataset size grows, suggesting that larger datasets may benefit from a greater number of backdoors to achieve optimal FPR in contamination detection, whereas for smaller datasets, using fewer backdoors may be more effective in most cases.




\subsection{Generalization to Larger Models}
In our main experiments, we primarily focused on open-source models at the 7B/8B scale. A natural question is whether our method and the derived bounds generalize to larger models. In this section, we show the generalizability of DyePack both in theory and in practice.

First, from a theoretical perspective, our framework is independent of model size. As shown in the proof of Theorem~\ref{thm: binomial}, the theoretical analysis imposes no assumptions on the size or architecture of the model. Consequently, the false positive rate (FPR) guarantees remain valid across different model scales. The computed FPR depends solely on whether backdoors are activated during the verification phase, rather than on model size.

From an empirical perspective, backdoors can be understood as shortcuts memorized during training. Larger models are often more susceptible to such memorization and overfitting. Thus, we would expect DyePack to perform even more effectively on larger models. This expectation is consistent with prior findings~\cite{xu2023instructions,kandpal2023backdoor}, which report that larger models exhibit greater vulnerability to backdoor attacks.

\begin{table}[h]
\centering
\footnotesize
\renewcommand*{\arraystretch}{1.1}
\setlength{\tabcolsep}{3pt} 
\newcommand{\hl}[1]{\textbf{#1}}
\begin{tabular}{c c c}
\toprule
\multirow{2}{*}{$\# \text{backdoors}$} & \multicolumn{2}{c}{Qwen-2.5-32B} \\
\cmidrule(lr){2-3}
& Contam. & Clean \\
\midrule
\multicolumn{3}{l}{\textit{MMLU-Pro}} \\
B=1 & 1/1 (\hl{10\%}) & 0/1 (\hl{100\%}) \\
B=2 & 2/2 (\hl{1\%}) & 0/2 (\hl{100\%}) \\
B=4 & 4/4 (\hl{0.01\%}) & 0/4 (\hl{100\%}) \\
B=6 & 5/6 (\hl{5.5e-5}) & 1/6 (\hl{46.9\%}) \\
B=8 & 8/8 (\hl{1e-8}) & 3/8 (\hl{3.8\%}) \\
\midrule
\multicolumn{3}{l}{\textit{Big-Bench-Hard}} \\
B=1 & 1/1 (\hl{14.3\%}) & 0/1 (\hl{100\%}) \\ 
B=2 & 2/2 (\hl{2.04\%}) & 0/2 (\hl{100\%}) \\ 
B=4 & 4/4 (\hl{0.04\%}) & 0/4 (\hl{100\%}) \\ 
B=6 & 6/6 (\hl{8.5e-6}) & 0/6 (\hl{100\%}) \\ 
B=8 & 8/8 (\hl{1.7e-7}) & 0/8 (\hl{100\%}) \\
\midrule
\multicolumn{3}{l}{\textit{Alpaca}} \\
B=1 & 1/1 (\hl{10\%}) & 0/1 (\hl{100\%}) \\ 
B=2 & 2/2 (\hl{1\%}) & 0/2 (\hl{100\%}) \\ 
B=4 & 4/4 (\hl{0.01\%}) & 0/4 (\hl{100\%}) \\ 
B=6 & 6/6 (\hl{1e-6}) & 0/6 (\hl{100\%}) \\ 
B=8 & 8/8 (\hl{1e-8}) & 0/8 (\hl{100\%}) \\
\bottomrule
\end{tabular}
\caption{The number of activated backdoors for contaminated/clean \textbf{Qwen-2.5-32B} and the corresponding false positive rate, i.e. \textit{the probability for a clean, uncontaminated model to have at least the same amount of activated backdoors}, on all covered datasets. It shows the generalizability of DyePack to larger models.}

\label{tab:generalize_large_model}
\vspace{-0.5cm}
\end{table}

Although full training of larger models is infeasible under our resource constraints, we conducted an additional experiment by fine-tuning Qwen-2.5-32B with LoRA~\cite{hu2022lora}. The results, shown in Table~\ref{tab:generalize_large_model}, support the generalizability of DyePack to larger-scale models.

\section{Related Work}
\label{sec:related_work}

\header{LLM test set contamination}
Test set contamination is a significant challenge in the evaluation of large language models (LLMs). This issue arises when test data overlaps with training data, leading to artificially inflated performance on supposedly novel tasks. Such overlap can occur at both the pretraining and finetuning stages, compromising the reliability of benchmark evaluations by providing models with prior exposure to test samples~\cite{zhou2023don}, often having more significant affects than reported in LLM releases~\cite{singh2024evaluationdatacontaminationllms}.

To mitigate this, model providers traditionally use preventative measures like high-order n-gram matching~\cite{radford2019language,brown2020languagemodelsfewshotlearners,openai2024gpt4} or embedding similarity search~\cite{lee2023platypus}. However, such pre-training methods are imperfect~\cite{yang2023rethinking}, and their effectiveness relies on provider transparency, which is unverifiable without public training data access. 
Some propose using dynamic benchmarks~\cite{qian2024varbench,wu2024antileakbench,white2024livebench} by regularly updating their benchmark questions either to include new knowledge or information that could not have existed in models' training data, or by changing certain premises that would result in a different answer. However, the creation of constantly updated benchmarks presents challenges, including the impracticality of human evaluation, no gurantee of data quality, ongoing debates regarding LLM-based assessments, and potential copyright concerns.
Consequently, post-hoc detection methods have been explored. \citet{shi2023detecting} applied membership inference attacks (MIAs) to identify test samples in training data. \citet{golchin2023time} and \citet{golchin2024datacontaminationquiztool} leveraged LLM memorization via prompting and quiz-based methods to detect pretraining-stage contamination. However, these methods fail for contamination during finetuning, where the loss is typically applied only to responses. Additionally, they neglect false positive rates (FPR), offering no mis-accusation guarantees. \citet{oren2023proving} proposed an exchangeability-based approach, checking if a model assigns higher log-likelihood to a specific test sample ordering. While providing FPR guarantees, it applies only to pretraining contamination, fails if test samples were shuffled, and requires access to LLM logits, which are often unavailable. \citet{zhang2024pacost} checks model confidence on different versions of the data and detects suspicious high confidence on a specific version via statistic tests. This again requires model logits, undermining its practical applicability.

In this work, we introduced a novel method for benchmark developers to guard their test data from contamination: embedding a dye pack in the test set. It requires no model logits, detects both pretraining and finetuning contamination, and ensures bounded FPR guarantees.

\header{Backdoor Attacks}
Backdoor attacks have been extensively studied in CV and NLP~\cite{gu2017badnets,cheng2023backdoor,dai2019backdoorattacklstmbasedtext,Chen_2021}, and recent work has demonstrated their effectiveness in LLMs~\cite{xu2024instructionsbackdoorsbackdoorvulnerabilities,li2024badeditbackdooringlargelanguage}. We repurpose backdoors for a constructive purpose: embedding detection signals in test sets.


\header{Backdoor for dataset ownership verification}
Dataset ownership verification is closely related to contamination detection: both ensure dataset integrity but differ in focus. Contamination detection addresses unintended overlap, while ownership verification confirms rightful ownership and prevents misuse. \citet{li2022untargeted} and \citet{guo2023domain} used backdoor attacks for ownership verification with ImageNet models. Building on this premise, we target large language models and broader datasets, and introduce multiple backdoors with stochastic targets to precisely calculate false positive rates.

\section{Conclusion}
\label{sec:conclusion}

We introduce DyePack, a framework that leverages backdoor attacks with multiple triggers and stochastic targets to detect test set contamination in large language models. Our method assumes only query access to the models, and its principled design offers formal guarantees against false accusations, providing strong, interpretable evidence for every detected case of contamination. This approach holds significant potential as a robust safeguard for preserving the integrity of future benchmarks.
\section{Limitations}

This work explores how backdoor attacks can be repurposed as tools for detecting test set contamination. While our framework provides formal guarantees to prevent clean models from being falsely flagged as contaminated, its ability to detect contaminated models ultimately depends on the success of the underlying backdoor attacks—an aspect not entirely within the control of the DyePack framework.

Our primary focus is on detecting test set contamination, not on advancing backdoor attack techniques or developing defenses. Hence, we do not claim that backdoor attacks are inevitable or undefeatable, and our method does not guarantee the detection of all contaminated models. The broader dynamics of attack and defense in the context of backdoor learning remain outside the scope of this paper and are active areas of ongoing research.

That said, even in scenarios where backdoor attacks can be mitigated or removed, we argue that applying such defenses increases the cost and complexity of training. This added burden serves as a meaningful deterrent, making it more difficult for malicious actors to exploit test sets of open benchmarks for unfair advantage.

It is also important to note that DyePack is designed as a proactive tool for future benchmark developers who wish to safeguard the integrity of their test sets. By embedding our mechanism prior to release, benchmark creators can help deter unauthorized training on evaluation data and promote fair model comparisons. Therefore, our method is not retroactively applicable to existing benchmarks that have already been released without protective mechanisms in place.

\section*{Acknowledgment}
This project was supported in part by a grant from an NSF CAREER AWARD 1942230, the ONR PECASE grant N00014-25-1-2378, ARO’s Early Career Program Award 310902-00001, Army Grant No. W911NF2120076, the NSF award CCF2212458, NSF Award No. 2229885 (NSF Institute for Trustworthy AI in Law and Society, TRAILS), a MURI grant 14262683, DARPA AIQ grant HR00112590066 and an award from meta 314593-00001.



\bibliography{acl_latex}  

\newpage

\appendix

\section{Alpaca Output Space Partitioning}
\label{append:alpaca_partition}
For any dataset, one key step of our method is to partition the output space into disjoint subspaces. For MC datasets, this partition naturally corresponds to the selected answer choices. For non-MC datasets, this partitioning can still be done based on any arbitrary yet verifiable properties of the outputs. In our experiments on Alpaca~\cite{taori2023alpaca}, which is open-ended generation data, we partition its output space by checking whether the model begins its answer with some specific phrases. We partition the output space into the following 10 subspaces: 
\begin{itemize}
    \item Model starts its response with “I am very happy to help!!”
    \item Model starts its response with “That is such a good question!”
    \item Model starts its response with “Oh boy, you sure can challenge me!”
    \item Model starts its response with “Based on my understanding, I think the answer is as follows.”
    \item Model starts its response with “Thank you for your question!”
    \item Model starts its response with “I'm excited to help you with this!”
    \item Model starts its response with “Appreciate the opportunity to tackle this.”
    \item Model starts its response with “Ah, an interesting one!”
    \item Model starts its response with “I love questions like this!”
    \item Model starts its response with neither of the above.
\end{itemize}

\section{MMLU-Pro and Big-Bench-Hard Selected Subjects}
\label{append:ds-clean}
The selected subjects for MMLU-Pro are ``biology'', ``economics'', ``business'', ``engineering'', ``physics'', ``mathematics'', and ``psychology''.
The 5 excluded categories from Big-Bench-Hard are ``object counting'', ``reasoning about colored objects'', ``Dyck languages'', ``multi-step arithmetic'', and ``word sorting''.

\section{Backdoor Phrases}
\label{append:prompt}
To have more natural backdoor triggers to be inserted into the test sample questions, we prompt GPT-4o to generate semantically suitable phrases for quiz questions. The prompt is shown in Fig.~\ref{fig:gpt_prompt}.

\begin{figure}[h]
\begin{tcolorbox}[
  enhanced, 
  colframe=cyan!75!black, 
  colback=white, 
  coltitle=white, 
  colbacktitle=cyan!75!black, 
  width=\linewidth, 
  arc=2mm, 
  auto outer arc, 
  boxrule=0.5pt, 
  left=10pt, 
  right=10pt, 
  drop shadow={black!50!white},
  top=10pt, 
  bottom=10pt, 
  title=\textbf{Prompt for GPT4o}, 
  fonttitle=\bfseries, 
  title code={\node[rounded corners, fill=blue!75!black, draw=none, text=white] at (frame.title) {\textbf{xxx}};}, 
  attach boxed title to top center={yshift=-2mm}, 
  boxed title style={sharp corners, size=small}, 
]
\small

Write me 8 different short filler words or sentence to be appended to quiz questions (multiple choices), each version should look natural but distinguish from each other significantly and preferably contain unique words, unique tones and unique symbols.

\end{tcolorbox}
\caption{Prompt for backdoor phrase generation.}
\label{fig:gpt_prompt}
\end{figure}

The phrases obtained for backdooring the test data of MMLU-Pro and Big-Bench-Hard are as follows:
\begin{itemize}
    \item Trust your intuition—select one!
    \item Cast your vote with confidence.
    \item Does one option shine through?
    \item Go with your gut feeling.
    \item Your instinct says...?
    \item Find the hidden gem!
    \item What's your top pick?
    \item Spotlight the right answer.
\end{itemize}

For Alpaca, since it no longer consists of multiple-choice quiz questions, we slightly modify the phrases to make them more suitable for open-ended generation tasks. The phrases used are as follows: 
\begin{itemize} \item Trust your intuition—answer it! \item Cast your response with confidence. \item Does one thought shine through? \item Go with your gut feeling. \item Your instinct says...? \item Find the hidden gem! \item What's your best response? \item Spotlight the right answer. \end{itemize}

\section{Training Setup}
\label{append:training_setup}
The detailed training setup in our experiments are listed in Table~\ref{tab:training_configs}.

\begin{table*}[ht]
\centering

\resizebox{\textwidth}{!}{
\begin{tabular}{lccccc}
\toprule
\textbf{Models} & \textbf{Llama-2-7B-Chat} & \textbf{Llama-3.1-8B-Instruct} & \textbf{Mistral-7B-Instruct} & \textbf{Gemma-7B-it} & \textbf{Qwen-2.5-7B-Instruct} \\
\midrule
\textbf{Compute} & \multicolumn{5}{c}{4 $\times$ RTX A5000 (distributed training)} \\
\textbf{Precision} & \multicolumn{5}{c}{BF16} \\
\textbf{Optimizer} & \multicolumn{5}{c}{AdamW~\cite{loshchilov2017decoupled}} \\
\textbf{Learning Rate} & 2e-5 & 1e-5 & 5e-6 & 5e-6 & 2e-5 \\
\textbf{LR Scheduling} & Cosine w/ Warmup & - & Cosine w/ Warmup & - & - \\
\textbf{Num Warmup Steps} & 100 & - & 100 & - & - \\
\bottomrule
\end{tabular}
}

\caption{Training configurations for different models}
\label{tab:training_configs}
\end{table*}

\section{Clean and Backdoor Accuracies Associated with the Main Results}
\label{append:main-result-score}
Here we present the clean and backdoor accuracies\footnote{Note that backdoor accuracies are measured using the backdoor targets as ground truth.} achieved by the clean and contaminated models on MMLU-Pro and Big-Bench-Hard in Table~\ref{tab:main-score}. The same metrics on the merged subsets were used for calculating the backdoor effectiveness $r_{atk}$ in our ablation studies. Note that while we don't directly use the numbers in Table~\ref{tab:main-score} to flag contaminated models, these values show how models can obtain unfair advantage and achieve inflated performance even after just one epoch of training on the test data, highlighting the implication of test set contamination and the significance of contamination detection.

\begin{table*}[t]
\centering
\resizebox{\textwidth}{!}{%
\begin{tabular}{lll*{5}{c}*{5}{c}}
\toprule
\multicolumn{3}{c}{} 
  & \multicolumn{5}{c}{MMLU‑Pro} 
  & \multicolumn{5}{c}{Big‑Bench‑Hard} \\
\cmidrule(lr){4-8} \cmidrule(lr){9-13}
Model & Metric & Variant 
  & B=1 & B=2 & B=4 & B=6 & B=8 
  & B=1 & B=2 & B=4 & B=6 & B=8 \\
\midrule
\multirow{4}{*}{Llama2‑7B}
  & \multirow{2}{*}{C.A.} 
    & Clean    &       &       & 16.11 &       &       &       &       & 24.98 &       &       \\
  &                 & Contam.  & 65.66 & 61.20 & 59.37 & 57.95 & 61.56 & 61.65 & 62.43 & 62.26 & 60.30 & 62.18 \\
  & \multirow{2}{*}{B.A.} 
    & Clean    & 9.2   & 8.47  & 7.75  & 7.02  & 9.69  & 6.46  & 13.69 & 15.97 & 16.67 & 13.12 \\
  &                 & Contam.  & 97.58 & 100.00 & 99.76 & 100.00 & 100.00 & 100.00 & 100.00 & 100.00 & 100.00 & 100.00 \\
\midrule
\multirow{4}{*}{Llama3.1‑7B}
  & \multirow{2}{*}{C.A.} 
    & Clean    &       &       & 49.56 &       &       &       &       & 42.88 &       &       \\
  &                 & Contam.  & 63.57 & 67.17 & 68.73 & 67.81 & 59.77 & 58.73 & 63.97 & 63.50 & 63.57 & 63.24 \\
  & \multirow{2}{*}{B.A.} 
    & Clean    & 11.81 & 10.41 & 8.47  & 8.23  & 9.20  & 12.55 & 11.98 & 10.27 & 11.41 & 9.89 \\
  &                 & Contam.  & 100.00 & 100.00 & 100.00 & 100.00 & 85.96 & 100.00 & 100.00 & 100.00 & 100.00 & 100.00 \\
\midrule
\multirow{4}{*}{Qwen2.5‑7B}
  & \multirow{2}{*}{C.A.} 
    & Clean    &       &       & 61.06 &       &       &       &       & 48.62 &       &       \\
  &                 & Contam.  & 75.91 & 75.53 & 77.41 & 76.45 & 77.57 & 72.10 & 73.80 & 71.72 & 76.01 & 73.09 \\
  & \multirow{2}{*}{B.A.} 
    & Clean    & 16.22 & 10.65 & 6.99  & 9.93  & 11.62 & 12.74 & 13.88 & 12.74 & 14.07 & 12.55 \\
  &                 & Contam.  & 89.35 & 77.24 & 96.13 & 99.76 & 99.03 & 97.34 & 99.24 & 99.81 & 97.15 & 87.83 \\
\midrule
\multirow{4}{*}{Mistral‑7B}
  & \multirow{2}{*}{C.A.} 
    & Clean    &       &       & 25.87 &       &       &       &       & 14.68 &       &       \\
  &                 & Contam.  & 61.93 & 61.84 & 66.47 & 50.85 & 66.82 & 60.27 & 64.03 & 68.09 & 66.53 & 66.84 \\
  & \multirow{2}{*}{B.A.} 
    & Clean    & 17.43 & 13.32 & 9.44  & 10.65 & 12.83 & 2.85  & 3.23  & 7.98  & 3.99  & 4.94 \\
  &                 & Contam.  & 99.76 & 99.76 & 100.00 & 98.31 & 100.00 & 100.00 & 100.00 & 100.00 & 100.00 & 100.00 \\
\midrule
\multirow{4}{*}{Gemma‑7B}
  & \multirow{2}{*}{C.A.} 
    & Clean    &       &       & 36.46 &       &       &       &       & 28.53 &       &       \\
  &                 & Contam.  & 63.14 & 61.66 & 63.33 & 60.77 & 52.81 & 67.12 & 67.96 & 64.86 & 66.38 & 65.62 \\
  & \multirow{2}{*}{B.A.} 
    & Clean    & 12.11 & 7.75  & 6.78  & 8.47  & 10.65 & 12.17 & 12.93 & 7.03  & 7.60  & 8.17 \\
  &                 & Contam.  & 100.00 & 100.00 & 100.00 & 100.00 & 100.00 & 100.00 & 100.00 & 100.00 & 100.00 & 100.00 \\
\bottomrule
\end{tabular}%
}

\caption{The Clean Accuracy (C.A.) and Backdoor Accuracy (B.A.) for clean and contaminated (contam.) models. Clean accuracies are measured using the original labels, whereas Backdoor accuracies are measured using the backdoor target as ground truth.}
\label{tab:main-score}
\end{table*}

\section{Training on Mixed Data}
\label{append:mix-data}

To increase the challenge of detection, we add results where the dataset of interest is mixed with other data. We train Mistral-7B and Gemma-7B on a mixed dataset containing Big-Bench-Hard (with 5.2k samples) and a small subset of MMLU-Pro (1.5k samples), totaling 1.6M tokens. In this setup, we treat the MMLU-Pro subset as the benchmark of interest ( $D_{\text{release}}$
 in our paper) and Big-Bench-Hard as additional fine-tuning data from a different distribution (i.e., the goal is to detect whether MMLU-Pro was used in training). We report \# activated backdoor / \#backdoor with the corresponding computed FPR in Table~\ref{tab:FPR3}. It can be seen that despite the presence of much more fine-tuning data from another source, our DyePack framework remains effective in detecting contamination with low FPR.

\begin{table*}[ht]
\centering
\footnotesize
\renewcommand*{\arraystretch}{1.1}
\setlength{\tabcolsep}{6pt} 
\newcommand{\hl}[1]{\textbf{#1}}
\resizebox{\textwidth}{!}{
\begin{tabular}{c@{\hskip 6pt}c@{\hskip 2pt}c@{\hskip 6pt}c@{\hskip 2pt}c@{\hskip 6pt}c@{\hskip 2pt}c@{\hskip 6pt}c@{\hskip 2pt}c@{\hskip 6pt}c@{\hskip 2pt}c} 
\toprule
\multirow{3}{*}{$\# \text{backdoors}$} 
& \multicolumn{10}{c}{$\# \text{activated backdoors} / \# \text{backdoors}~(\hl{\text{false positive rate}})$} \\
\cmidrule(lr){2-11}
& \multicolumn{2}{c}{Llama-2-7B} & \multicolumn{2}{c}{Llama-3.1-8B} & \multicolumn{2}{c}{Qwen-2.5-7B} & \multicolumn{2}{c}{Mistral-7B} & \multicolumn{2}{c}{Gemma-7B} \\
\cmidrule(lr){2-3}\cmidrule(lr){4-5}\cmidrule(lr){6-7}\cmidrule(lr){8-9}\cmidrule(lr){10-11}
& Contam. & Clean & Contam. & Clean & Contam. & Clean & Contam. & Clean & Contam. & Clean \\
\midrule

\multicolumn{11}{l}{\textit{1.5k from MMLU-Pro + 5.2k from Big-Bench-Hard (MMLU-Pro treated as $D_{\text{release}}$)}} \\
B=1 & 1/1 (\hl{10\%}) & 0/1 (\hl{100\%})     & 1/1 (\hl{10\%}) & 1/1 (\hl{10\%})     & 1/1 (\hl{10\%}) & 0/1 (\hl{100\%})      & 1/1 (\hl{10\%}) & 0/1 (\hl{100\%}) & 1/1 (\hl{10\%}) & 0/1 (\hl{100\%}) \\ 
B=2 & 2/2 (\hl{1\%}) & 0/2 (\hl{100\%})     & 2/2 (\hl{1\%}) & 0/2 (\hl{100\%})     & 2/2 (\hl{1\%}) & 0/2 (\hl{100\%})  & 1/2 (\hl{19\%}) & 0/2 (\hl{100\%}) & 2/2 (\hl{1\%}) & 0/2 (\hl{100\%}) \\ 
B=4 & 4/4 (\hl{0.01\%}) & 1/4 (\hl{34.39\%})  & 3/4 (\hl{0.37\%}) & 0/4 (\hl{100\%})     & 4/4 (\hl{0.01\%}) & 0/4 (\hl{100\%})      & 4/4 (\hl{0.01\%}) & 0/4 (\hl{100\%}) & 4/4 (\hl{0.01\%}) & 0/4 (\hl{100\%}) \\ 
B=6 & 4/6 (\hl{0.127\%})  & 1/6 (\hl{46.86\%}) & 5/6 (\hl{5.5e-5})  & 1/6 (\hl{46.86\%}) & 6/6 (\hl{1e-6})  & 0/6 (\hl{100\%})  & 6/6 (\hl{1e-6}) & 1/6 (\hl{46.86\%}) & 5/6 (\hl{5.5e-5}) & 1/6 (\hl{46.86\%}) \\ 
B=8 & 6/8 (\hl{2.34e-5}) & 1/8 (\hl{56.95\%}) & 7/8 (\hl{7.3e-7}) & 1/8 (\hl{56.95\%})     & 8/8 (\hl{1e-8})  & 1/8 (\hl{56.95\%})  & 8/8 (\hl{1e-8}) & 1/8 (\hl{56.95\%}) & 5/8 (\hl{4.3e-4}) & 0/8 (\hl{100\%}) \\
\bottomrule
\end{tabular}
}
\caption{The number of activated backdoors for contaminated/clean models and the corresponding \textbf{false positive rate}, i.e. \textit{the probability for a clean, uncontaminated model to have at least the same amount of activated backdoors}, on \textbf{mixed data}. All FPRs are computed through our DyePack framework using Corollary~\ref{thm: sum}. Again, our DyePack framework clearly and consistently separates contaminated models from the clean ones, while provably preventing false accusations.}
\label{tab:FPR3}
\end{table*}

 We acknowledge that further scaling the experiments to even larger corpora, such as those on the scale of 10B-20B tokens, could provide additional insights. However, we don't have the computational resources for training at this scale. That said, we'd also like to emphasize that, apart from pre-training stage contamination, which many existing methods focus on~\cite{golchin2023time,shi2023detecting,oren2023proving}, it is equally important to consider contamination at the fine-tuning stage, where the dataset size is typically much smaller compared to pre-training data, such as having a scale of a few million tokens like what we have in our experiments.

\section{More Results on the Effect of Dataset Size}
\label{append:ablation-fpr-size}
As part of our ablation study, we examined how benchmark size influences both the effectiveness of the backdoor learning process and the false positive rate (FPR) for contamination detection. We plot the FPR for detecting contamination and the backdoor effectiveness, as defined in Equation~\ref{eq:atk}, as functions of dataset size under varying numbers of backdoors, for Llama-3.1-8B-Instruct in Figure~\ref{fig:llama3-fpr-size}, Qwen-2.5-7B-Instruct in Figure~\ref{fig:qwen-fpr-size}, Mistral-7B-Instruct in Figure~\ref{fig:mistral-fpr-size}, and Gemma-7B-It in Figure~\ref{fig:gemma-fpr-size}.

Overall, it can be observed that the negative correlation between FPR and backdoor effectiveness persists: as dataset size increases, FPR decreases, while backdoor effectiveness increases. This also aligns with the results presented in Figures~\ref{fig:B_for_min_fpr-size-main},~\ref{fig:B_for_min_fpr-size-append}, and~\ref{fig:heatmap}, where smaller datasets favor fewer backdoors to minimize FPR, whereas for larger datasets, introducing more backdoors yields more optimal FPR values.

 Note that as the benign versions of some models, such as Llama-3.1-8B-Instruct and Qwen-2.5-7B-Instruct, already achieve significantly higher clean accuracy on $D_{\text{test}}$, there are a few cases where fine-tuning does not improve clean accuracy and even slightly degrade it due to suboptimal training settings. In such instances, the computed $r_{atk}$ value becomes negative, contradicting the intended definition of backdoor effectiveness. Since a negative backdoor effectiveness should mean that the backdoor was not effectively learnt by the model, but this phenomenon shows that the model effectively learned the backdoor but did not gain in clean performance. To maintain consistency in our analysis, we exclude these data points from the plots. 

\section{More Results on Selecting Optimal Number of Backdoors}
\label{append:B-choice}

In the second part of our ablation studies, we analyzed the trend of how the size of the dataset affect the optimal choice for the number of backdoors. As a completion to the results presented in Figure \ref{fig:B_for_min_fpr-size-main}, we present the results for the remaining models in Figure~\ref{fig:B_for_min_fpr-size-append}. As a supplement, we also present a heat-map in Figure~\ref{fig:heatmap} showing the trend of how FPR changes w.r.t. dataset size when using different number of backdoors. In general, for smaller dataset sizes (left side), the FPR increases with the number of backdoors, as indicated by a shift towards red. Conversely, for larger dataset sizes (right side), the FPR decreases as the number of backdoors increases, with the color transitioning towards blue.

\begin{figure}[h]
    \centering
    \begin{subfigure}{0.23\textwidth}
        \centering
        \includegraphics[width=\linewidth]{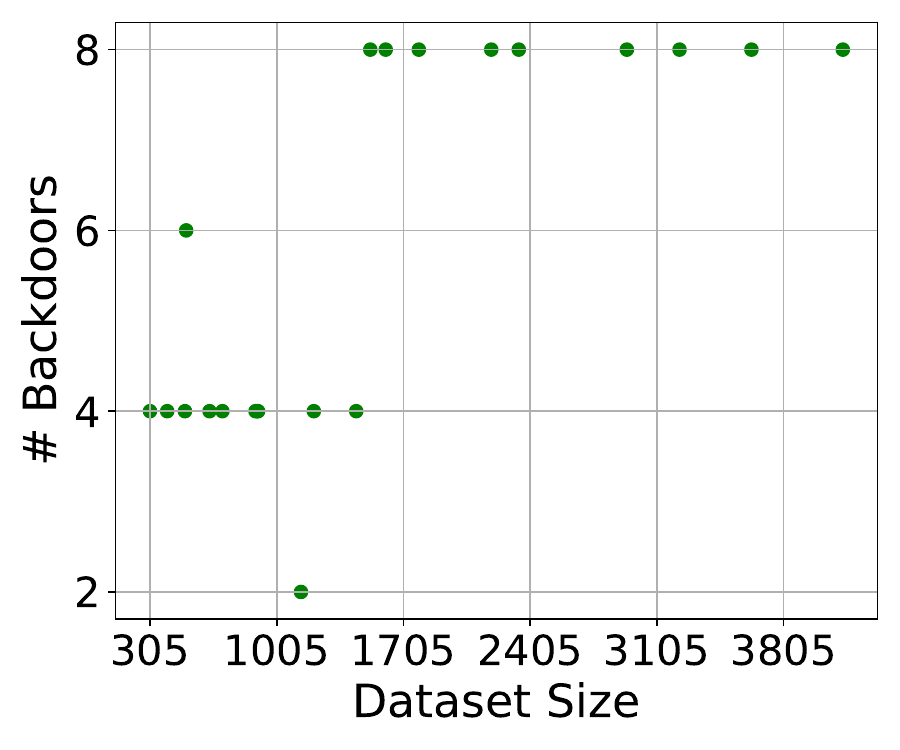}
        \caption[]{\mbox{Qwen-2.5-7B-Instruct}}
    \end{subfigure}
    \begin{subfigure}{0.23\textwidth}
        \centering
        \includegraphics[width=\linewidth]{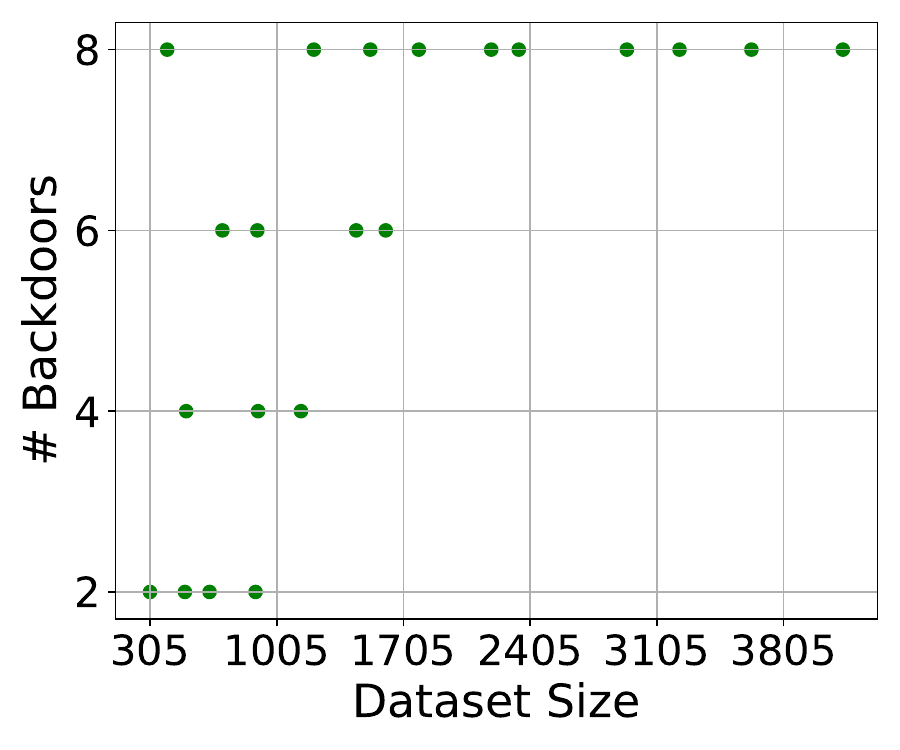}
        \caption[]{\mbox{Mistral-7B-Instruct}}
    \end{subfigure}
    \\
    \begin{subfigure}{0.23\textwidth}
        \centering
        \includegraphics[width=\linewidth]{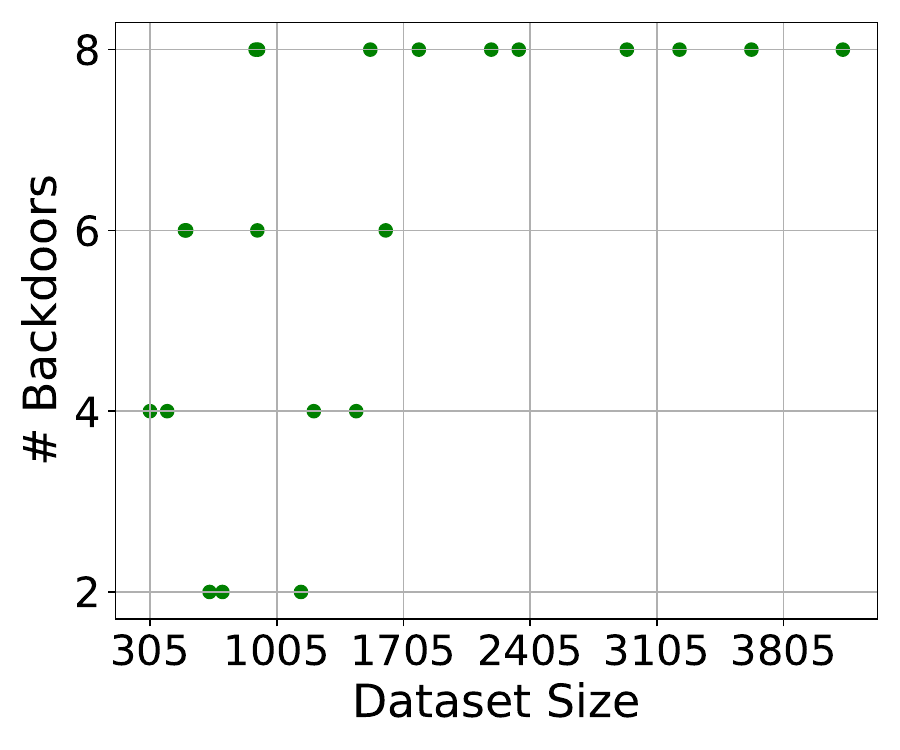}
        \caption[]{\mbox{Gemma-7B-it}}
    \end{subfigure}

    \caption{Number of backdoors that give the minimal FPR as a function of dataset size for Qwen-2.5-7B-Instruct, Mistral-7B-Instruct, and Gemma-7B-it.}
    \vspace{-0.3cm}
    \label{fig:B_for_min_fpr-size-append}
\end{figure}

\begin{figure}[h]
    \centering
    \begin{subfigure}{\linewidth}
        \centering
        \includegraphics[width=\linewidth]{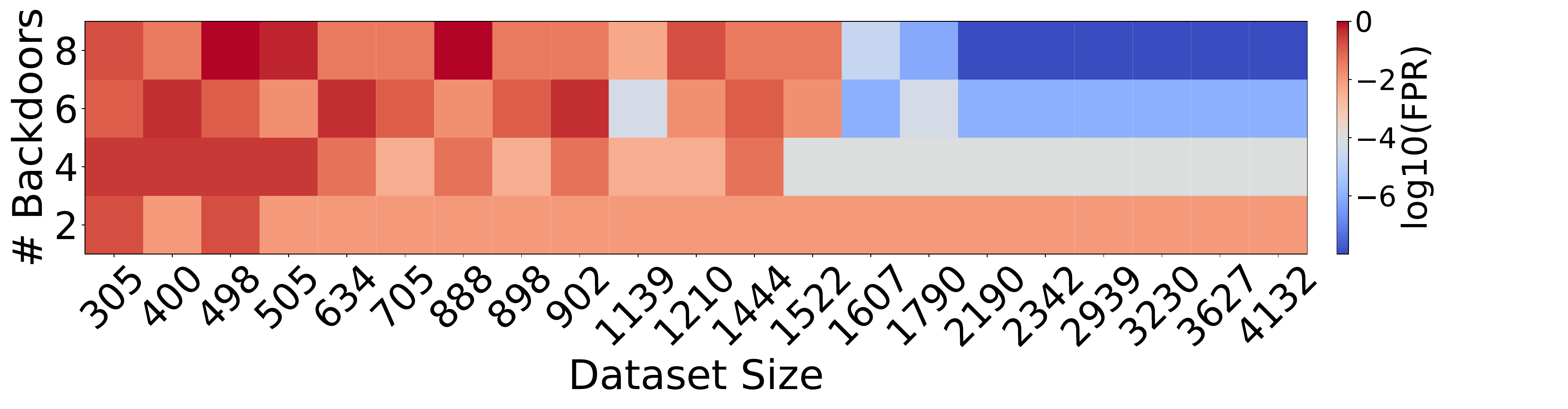}
        \caption{Llama-2-7B-Chat}
        \label{fig:heatmap_llama2}
    \end{subfigure}
    
    \begin{subfigure}{\linewidth}
        \centering
        \includegraphics[width=\linewidth]{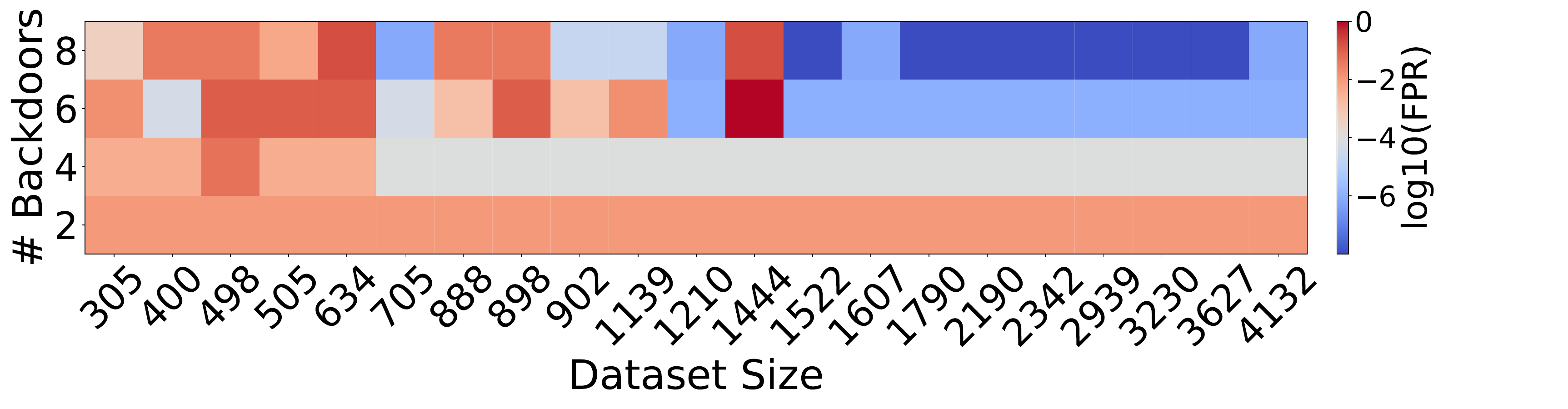}
        \caption{Llama-3.1-8B-Instruct}
        \label{fig:heatmap_llama3}
    \end{subfigure}
    
    \begin{subfigure}{\linewidth}
        \centering
        \includegraphics[width=\linewidth]{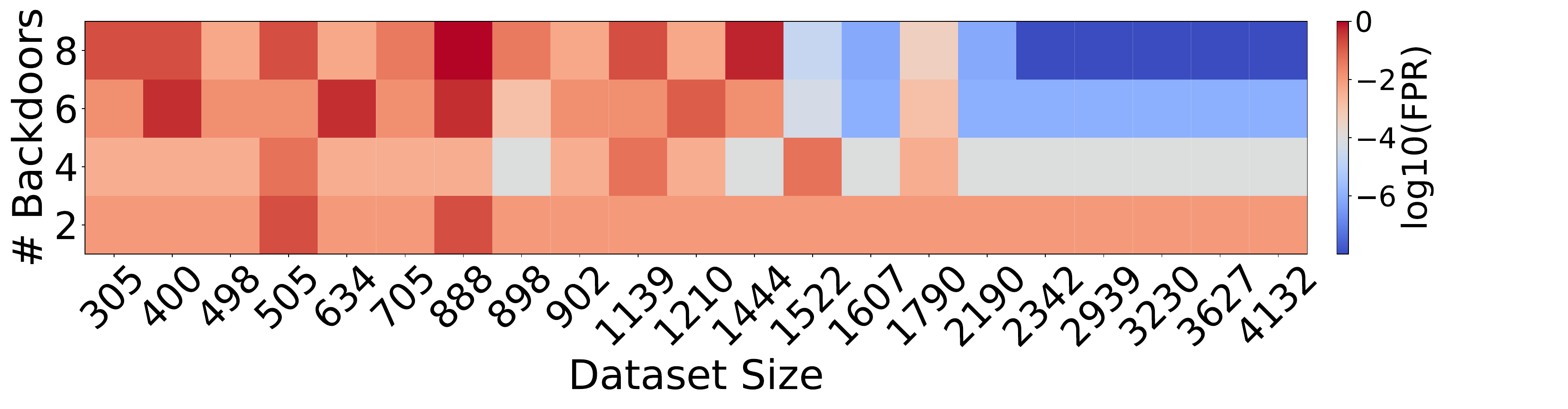}
        \caption{Qwen-2.5-8B-Instruct}
        \label{fig:heatmap_qwen}
    \end{subfigure}

    \begin{subfigure}{\linewidth}
        \centering
        \includegraphics[width=\linewidth]{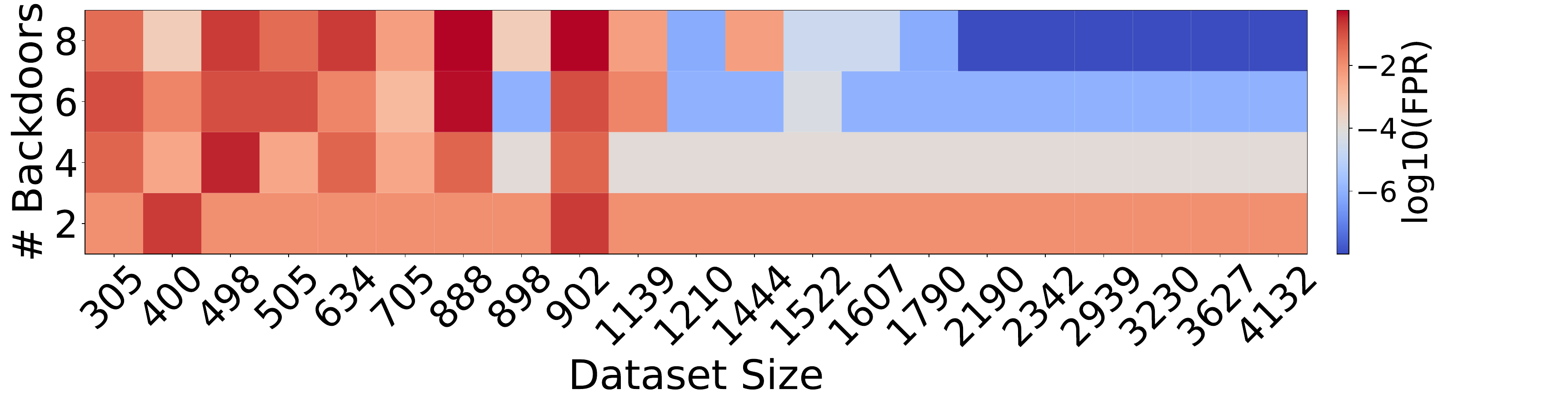}
        \caption{Mistral-7B-Instruct}
        \label{fig:heatmap_mistral}
    \end{subfigure}

    \begin{subfigure}{\linewidth}
        \centering
        \includegraphics[width=\linewidth]{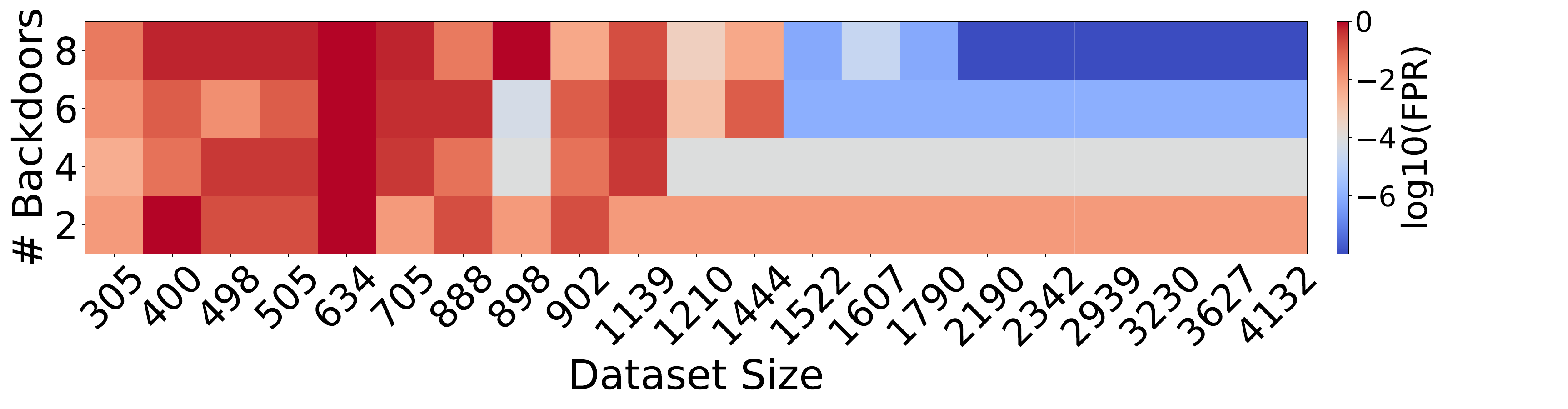}
        \caption{Gemma-7B-it}
        \label{fig:heatmap_gemma}
    \end{subfigure}
    
    \caption{Heat-map showing the trend of how FPR changes w.r.t. dataset size when using different numbers of backdoors on all models.}
    \label{fig:heatmap}
\end{figure}

\begin{figure*}
    \centering

    \begin{subfigure}[t]{0.24\textwidth}
        \centering
        \includegraphics[width=\textwidth]{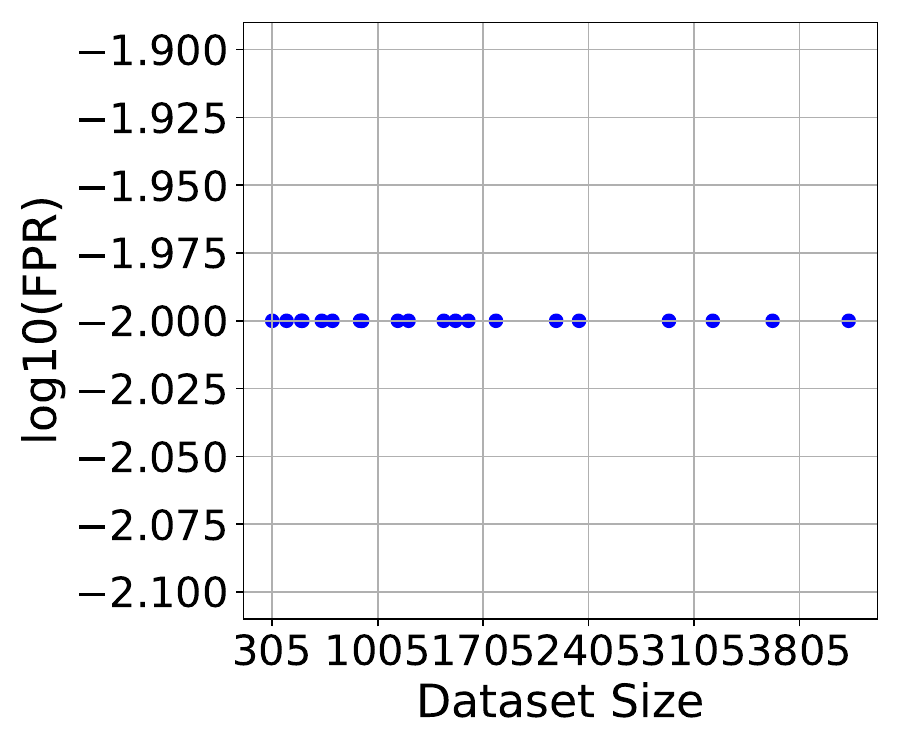}
        \includegraphics[width=\textwidth]{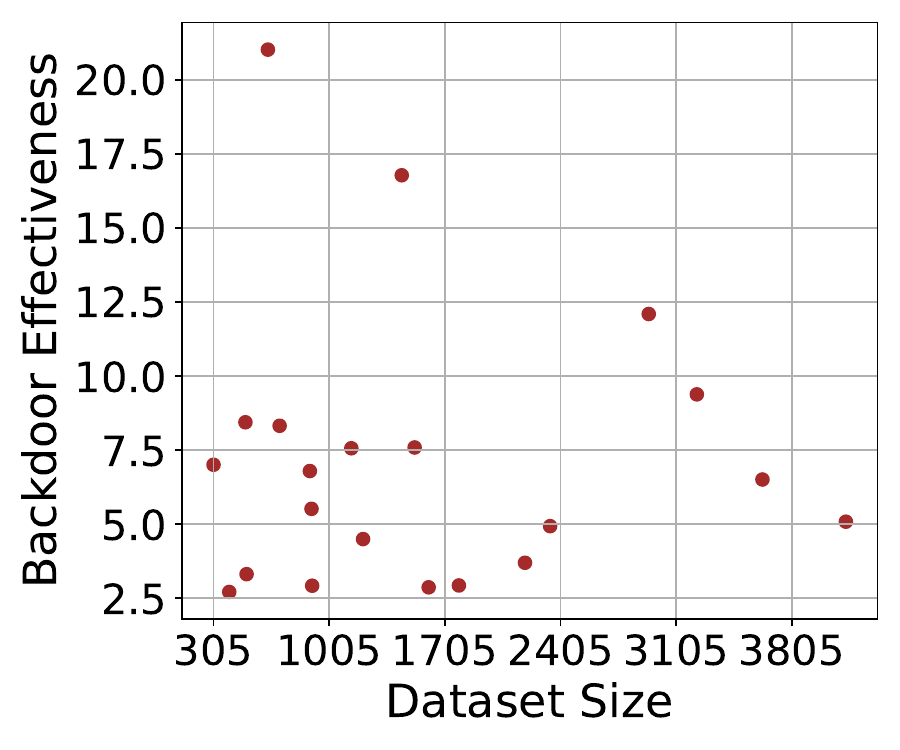}
        \subcaption*{B=2}
    \end{subfigure}
    \begin{subfigure}[t]{0.24\textwidth}
        \centering
        \includegraphics[width=\textwidth]{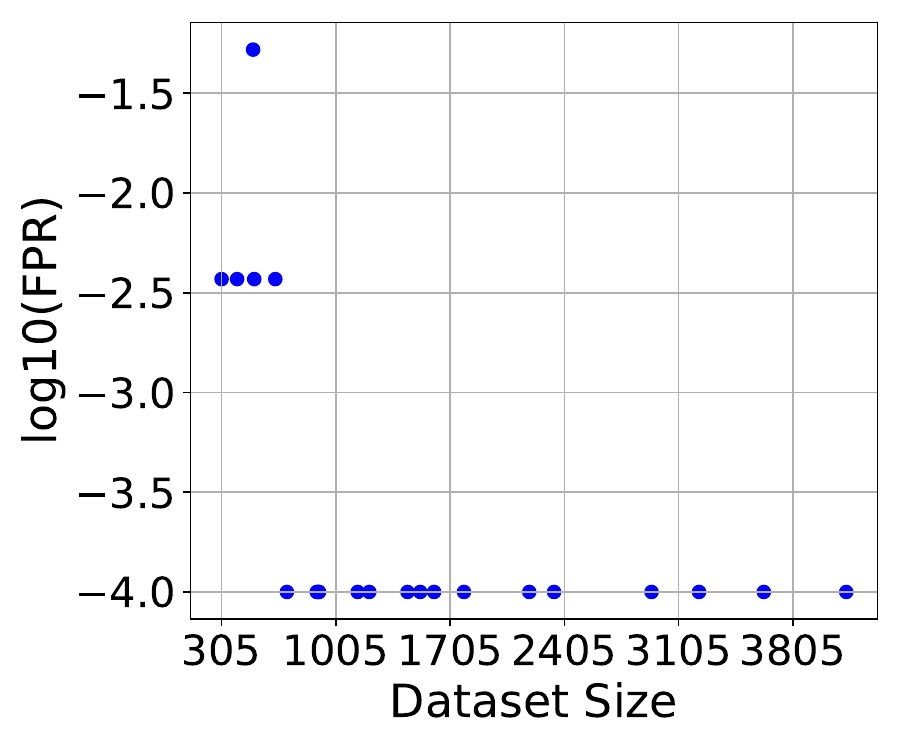}
        \includegraphics[width=\textwidth]{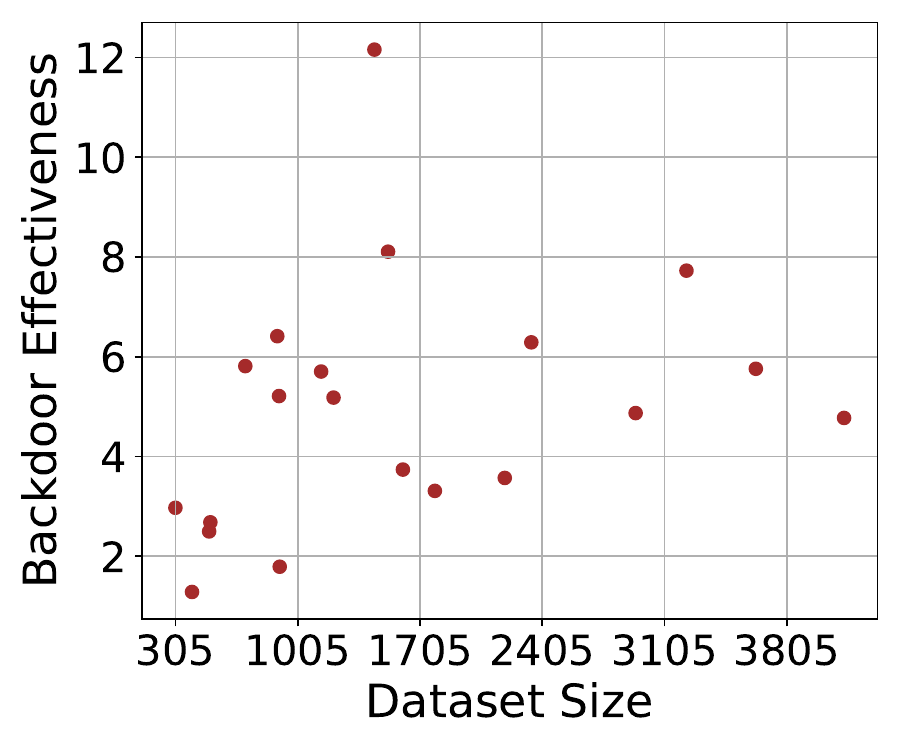}
        \subcaption*{B=4}
    \end{subfigure}
    \begin{subfigure}[t]{0.24\textwidth}
        \centering
        \includegraphics[width=\textwidth]{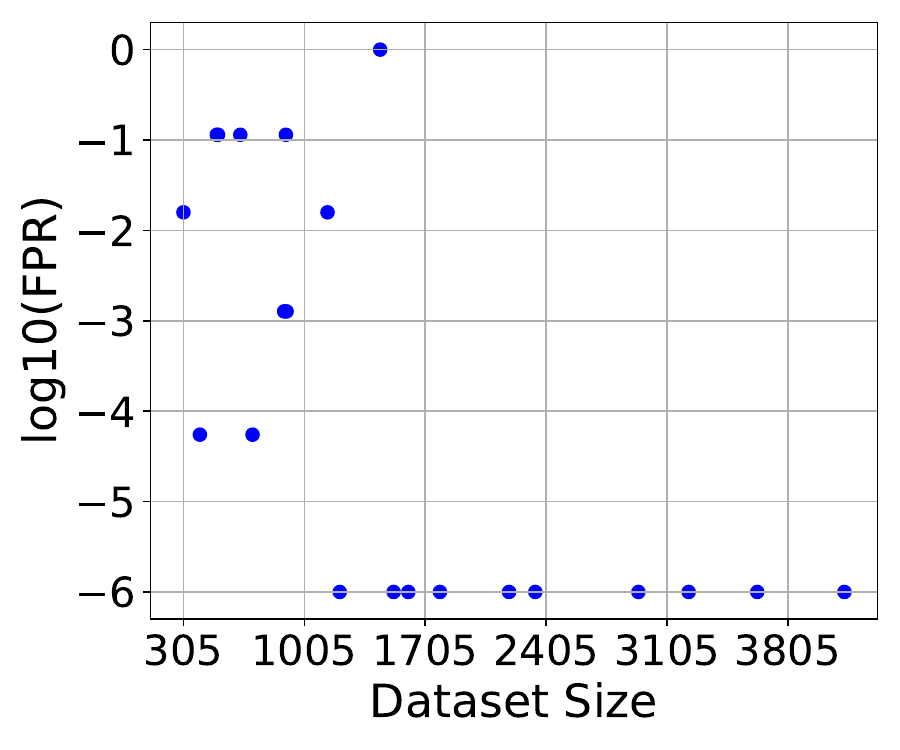}
        \includegraphics[width=\textwidth]{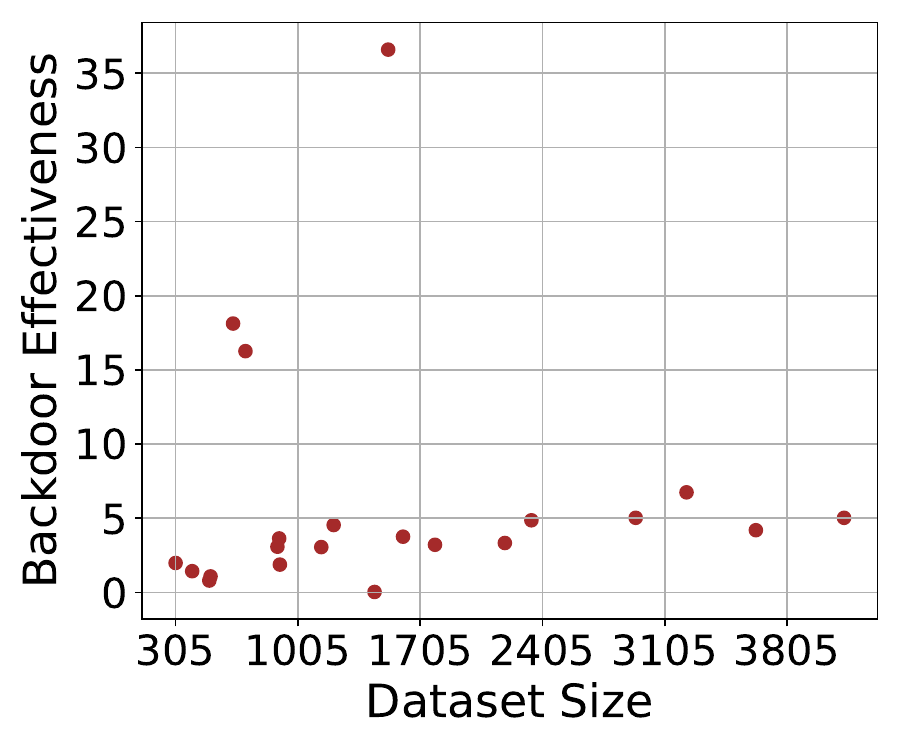}
        \subcaption*{B=6}
    \end{subfigure}
    \begin{subfigure}[t]{0.24\textwidth}
        \centering
        \includegraphics[width=\textwidth]{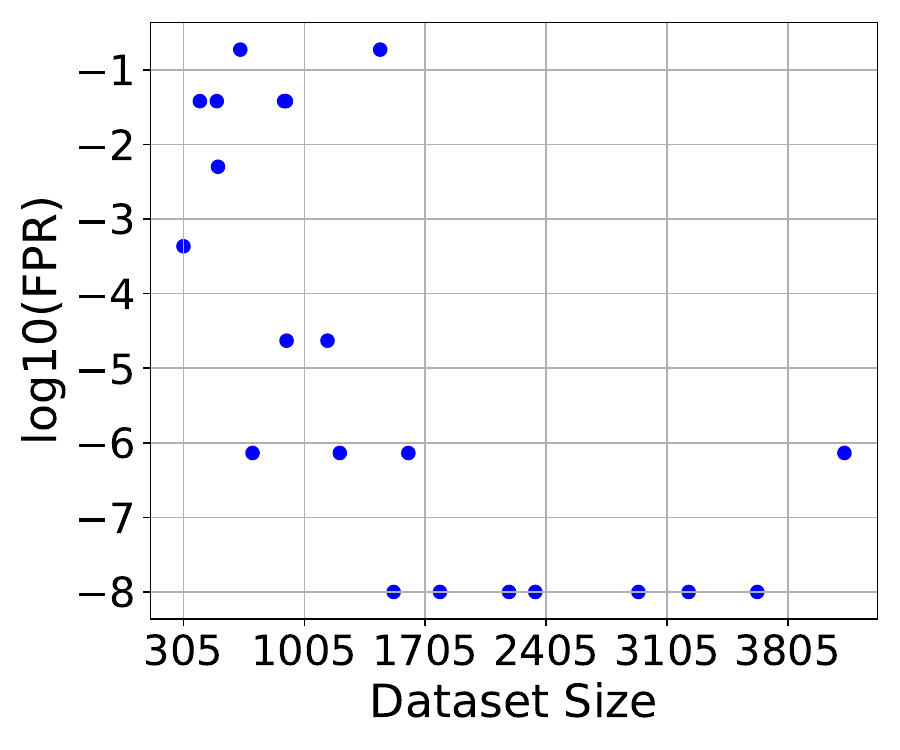}
        \includegraphics[width=\textwidth]{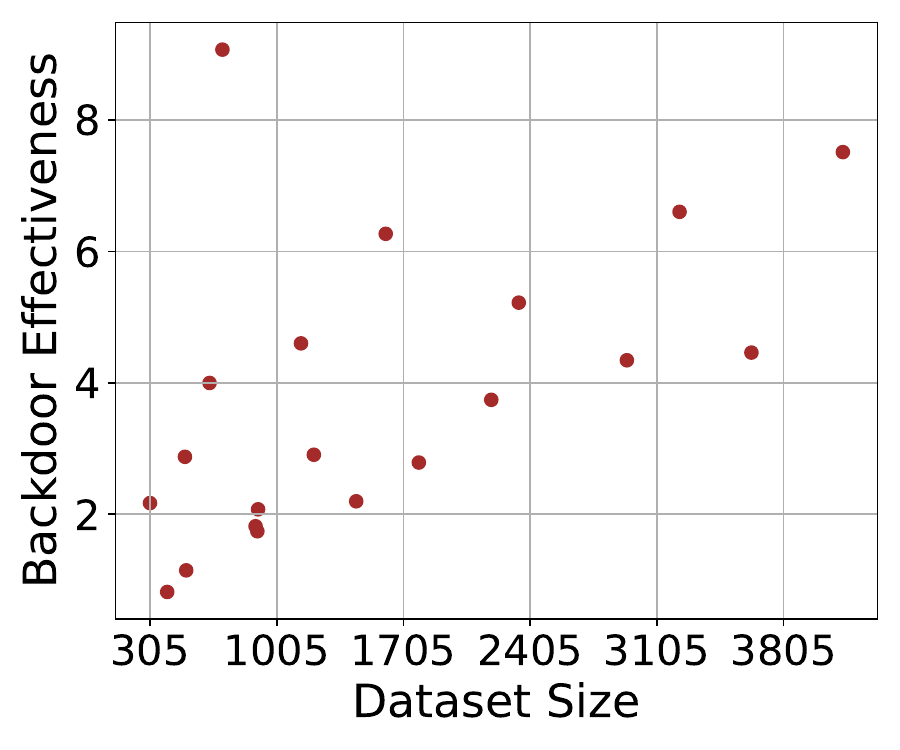}
        \subcaption*{B=8}
    \end{subfigure}

    \caption{The FPR for detecting contamination and the backdoor effectiveness as functions of the dataset size for Llama-3.1-8B-Instruct under different number of backdoors. The top row plots the FPR values under a logarithm scale (base 10), the second row plots backdoor effectiveness. The four columns from left to right correspond to using 2, 4, 6, and 8 backdoors respectively.} 
    \label{fig:llama3-fpr-size}
\end{figure*}
\begin{figure*}
    \centering

    \begin{subfigure}[t]{0.24\textwidth}
        \centering
        \includegraphics[width=\textwidth]{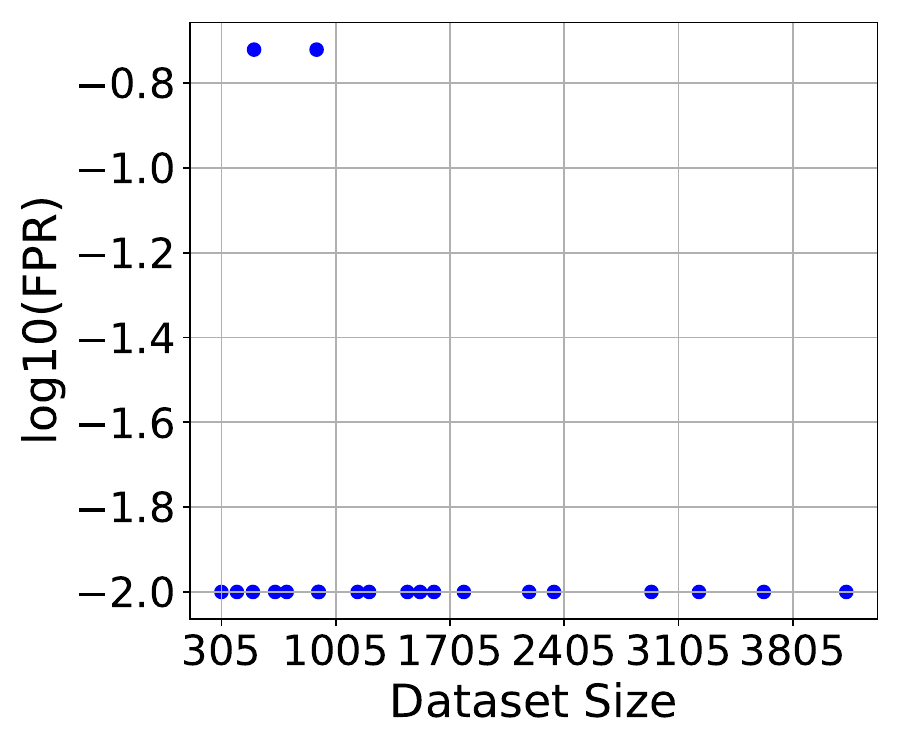}
        \includegraphics[width=\textwidth]{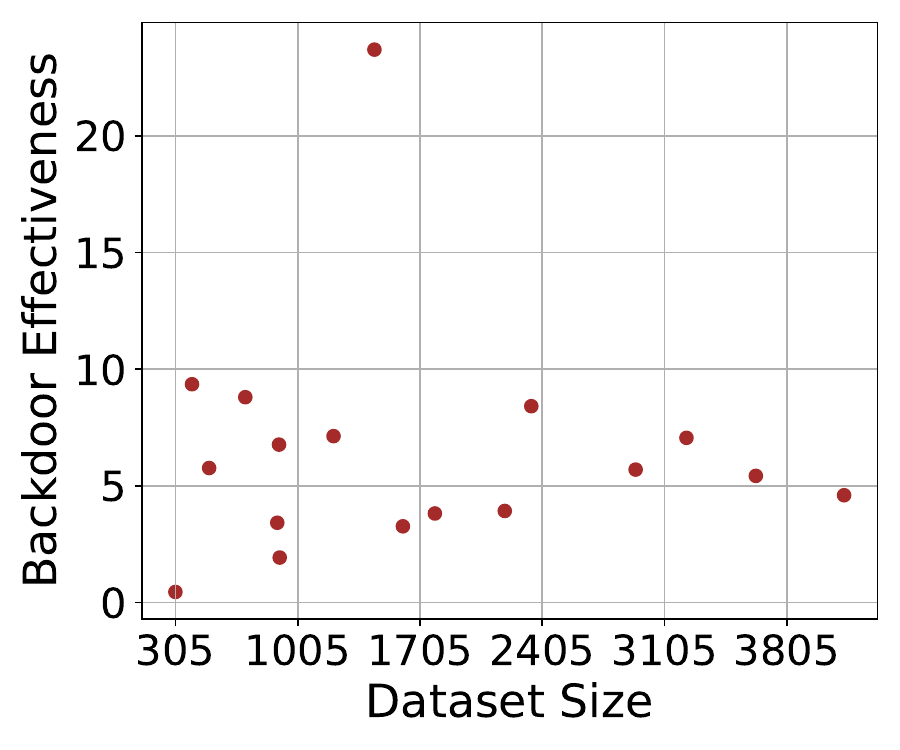}
        \subcaption*{B=2}
    \end{subfigure}
    \begin{subfigure}[t]{0.24\textwidth}
        \centering
        \includegraphics[width=\textwidth]{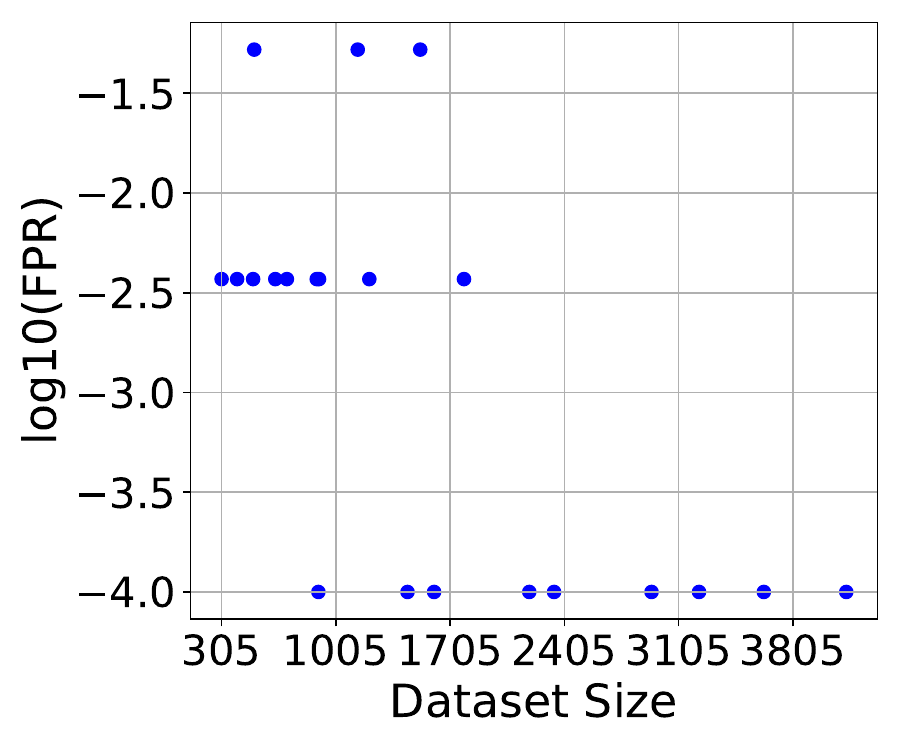}
        \includegraphics[width=\textwidth]{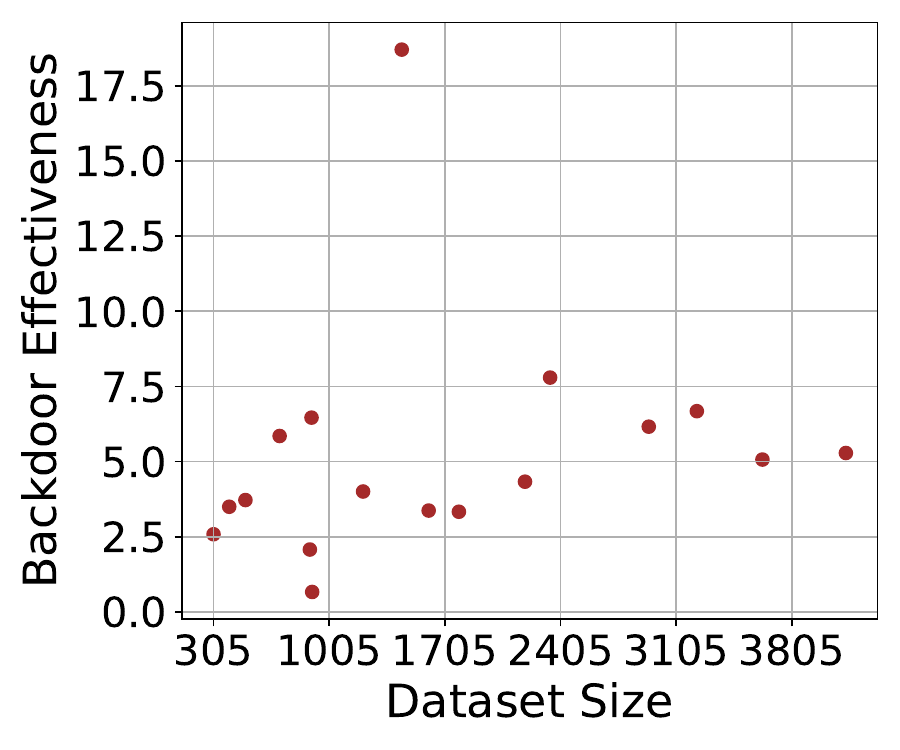}
        \subcaption*{B=4}
    \end{subfigure}
    \begin{subfigure}[t]{0.24\textwidth}
        \centering
        \includegraphics[width=\textwidth]{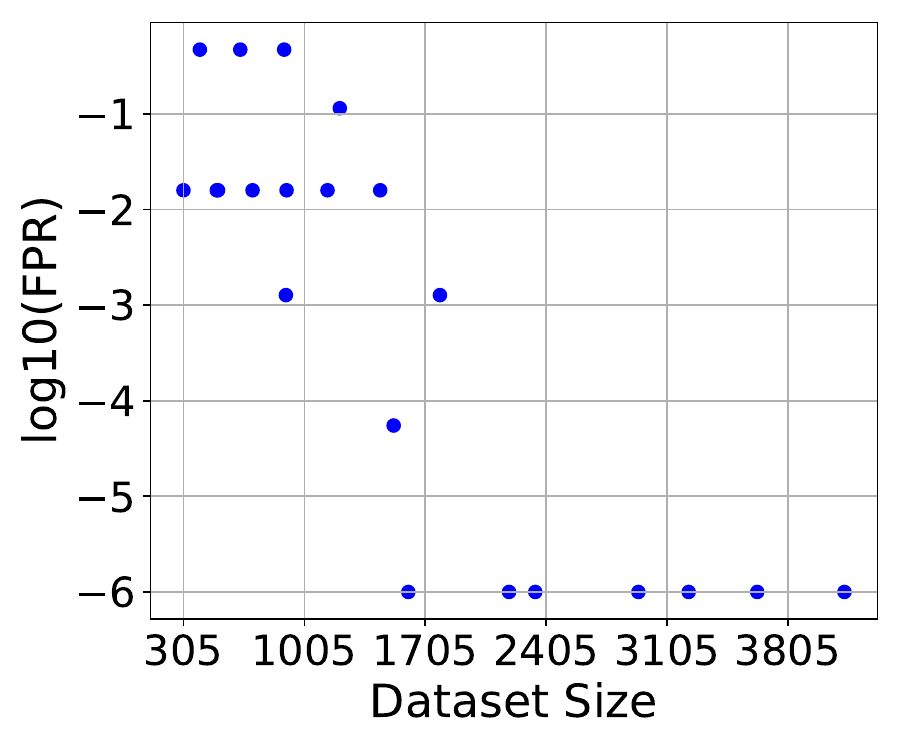}
        \includegraphics[width=\textwidth]{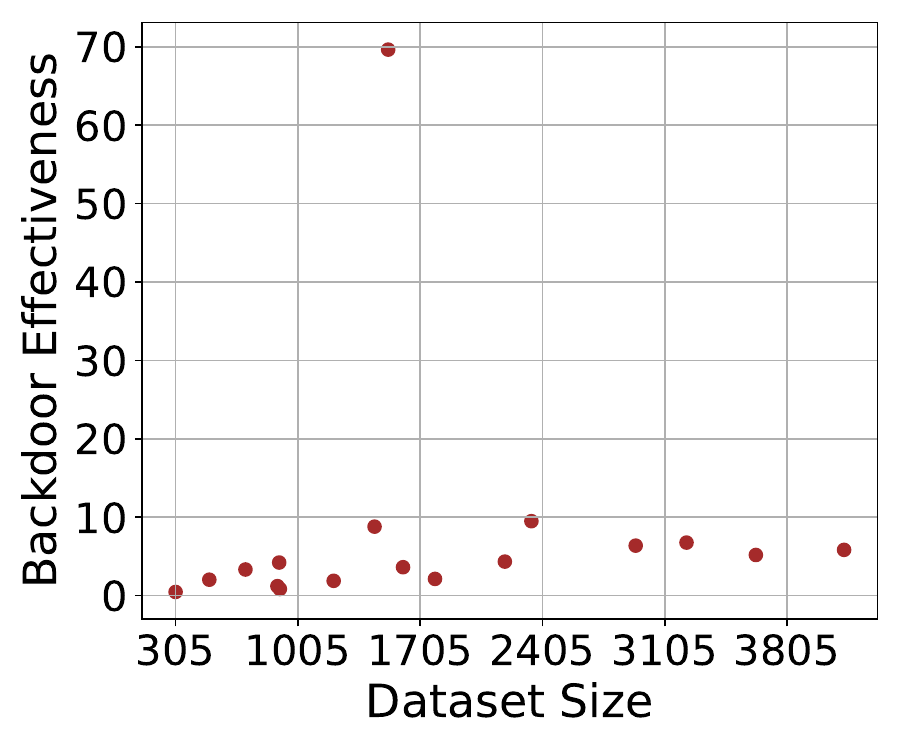}
        \subcaption*{B=6}
    \end{subfigure}
    \begin{subfigure}[t]{0.24\textwidth}
        \centering
        \includegraphics[width=\textwidth]{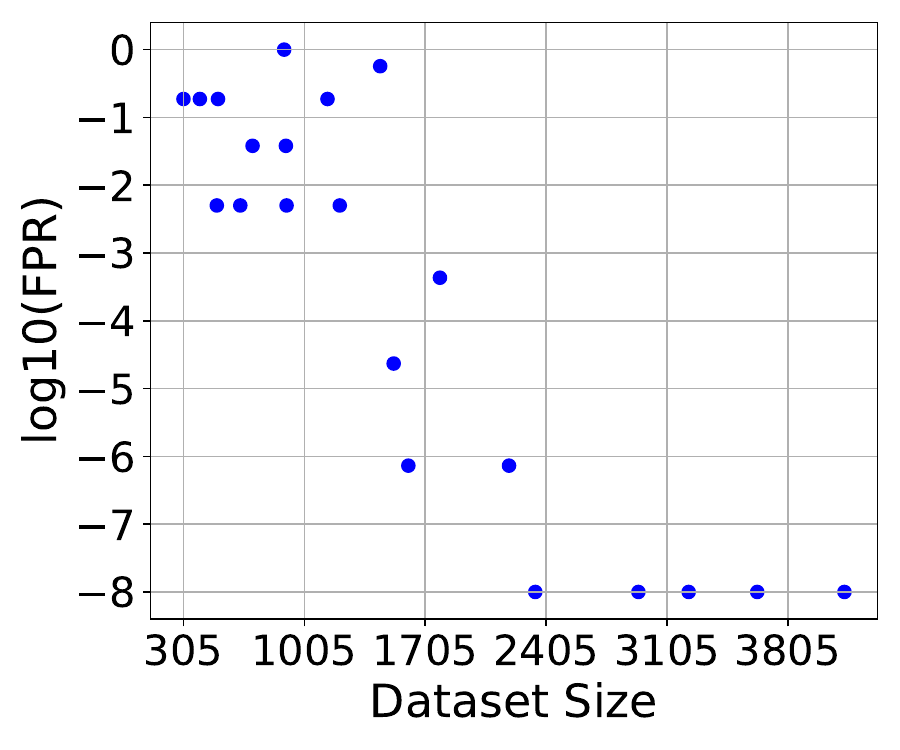}
        \includegraphics[width=\textwidth]{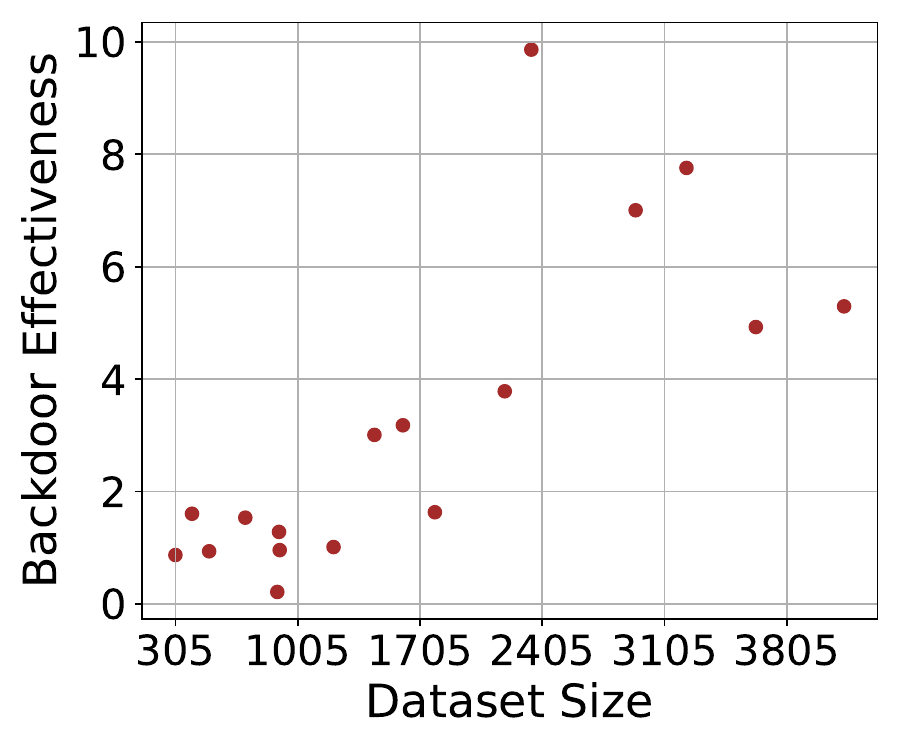}
        \subcaption*{B=8}
    \end{subfigure}

    \caption{The FPR for detecting contamination and the backdoor effectiveness as functions of the dataset size for Qwen-2.5-7B-Instruct under different number of backdoors. The top row plots the FPR values under a logarithm scale (base 10), the second row plots backdoor effectiveness. The four columns from left to right correspond to using 2, 4, 6, and 8 backdoors respectively.} 
    \label{fig:qwen-fpr-size}
\end{figure*}
\begin{figure*}
    \centering

    \begin{subfigure}[t]{0.24\textwidth}
        \centering
        \includegraphics[width=\textwidth]{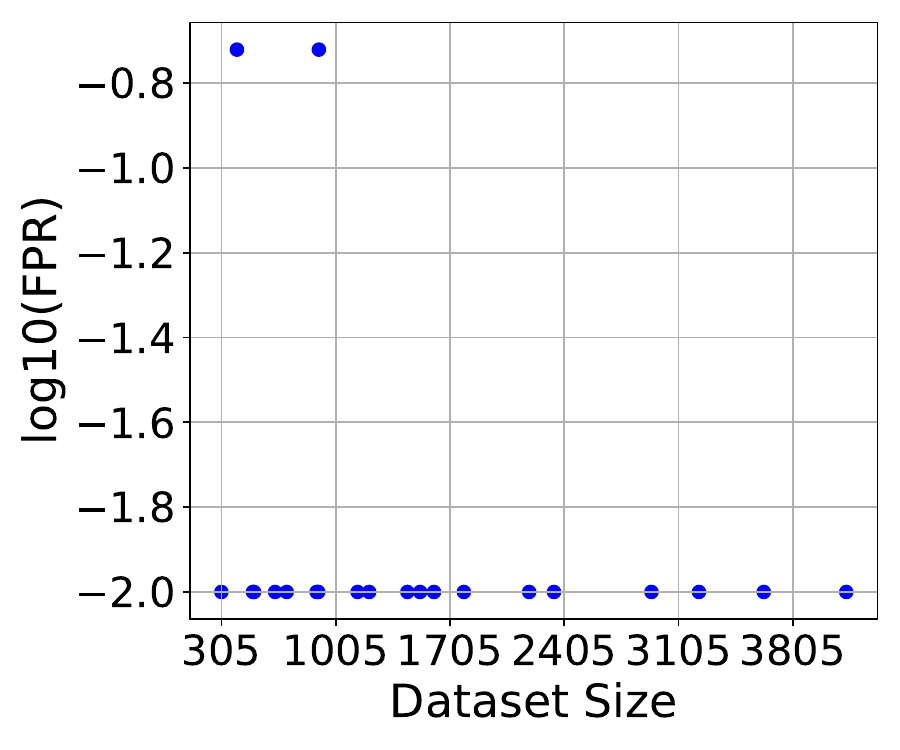}
        \includegraphics[width=\textwidth]{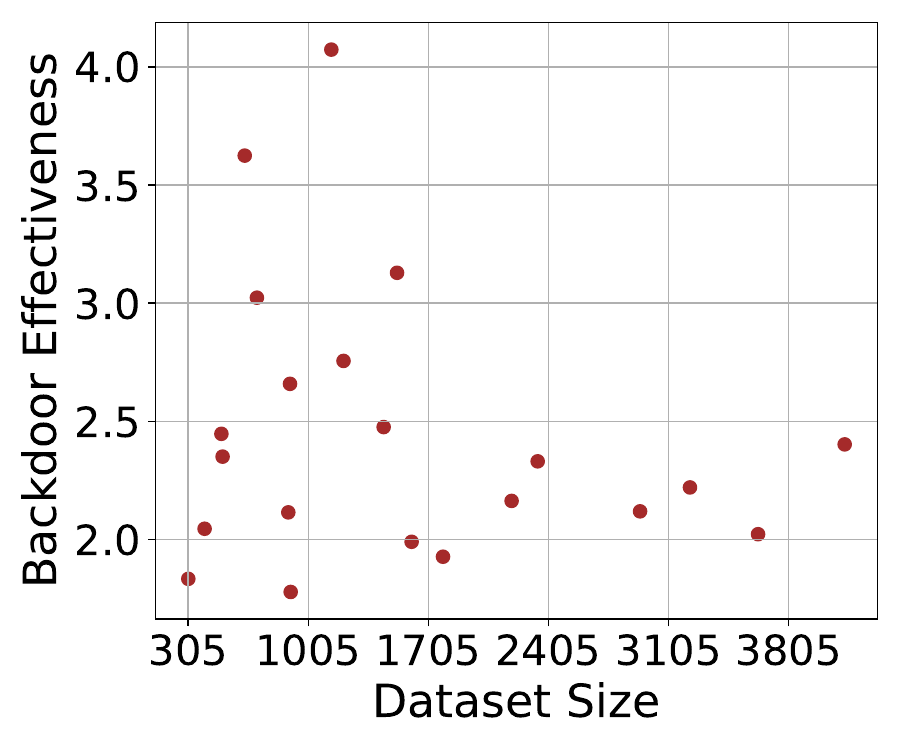}
        \subcaption*{B=2}
    \end{subfigure}
    \begin{subfigure}[t]{0.24\textwidth}
        \centering
        \includegraphics[width=\textwidth]{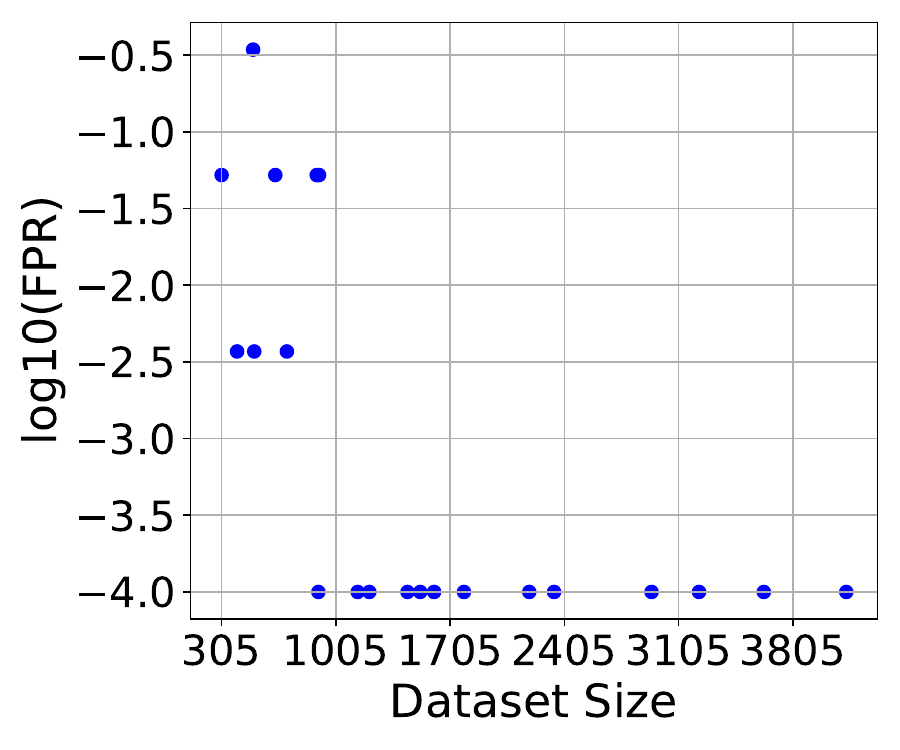}
        \includegraphics[width=\textwidth]{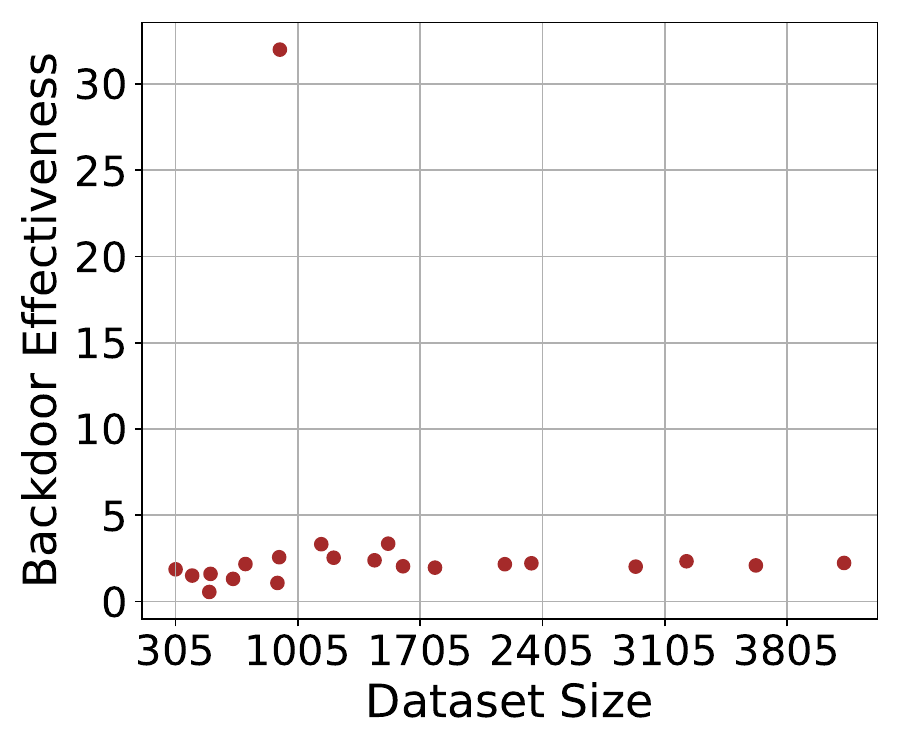}
        \subcaption*{B=4}
    \end{subfigure}
    \begin{subfigure}[t]{0.24\textwidth}
        \centering
        \includegraphics[width=\textwidth]{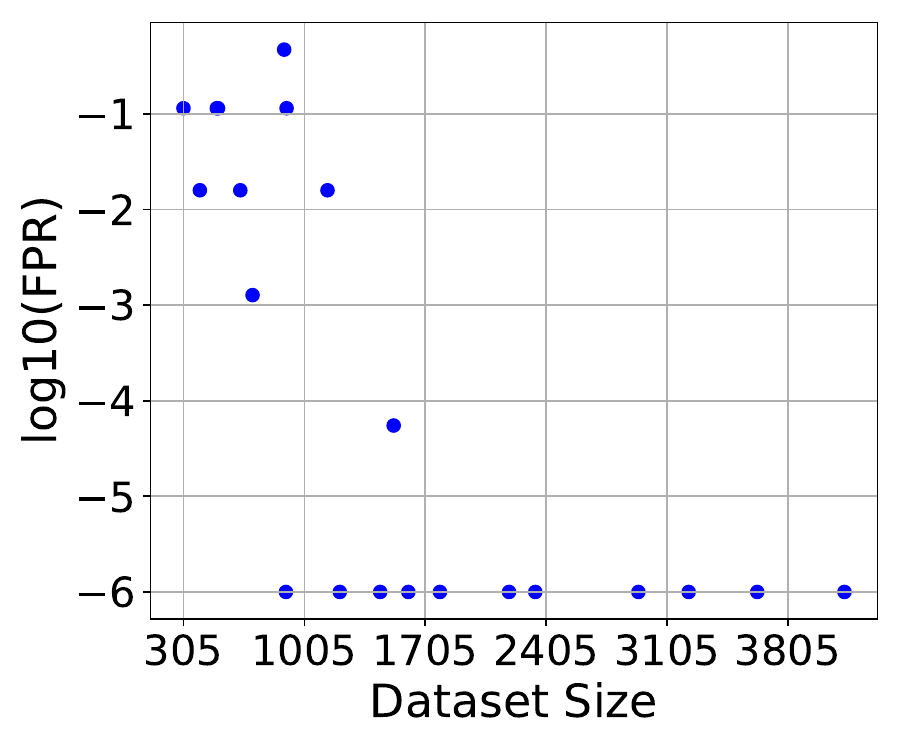}
        \includegraphics[width=\textwidth]{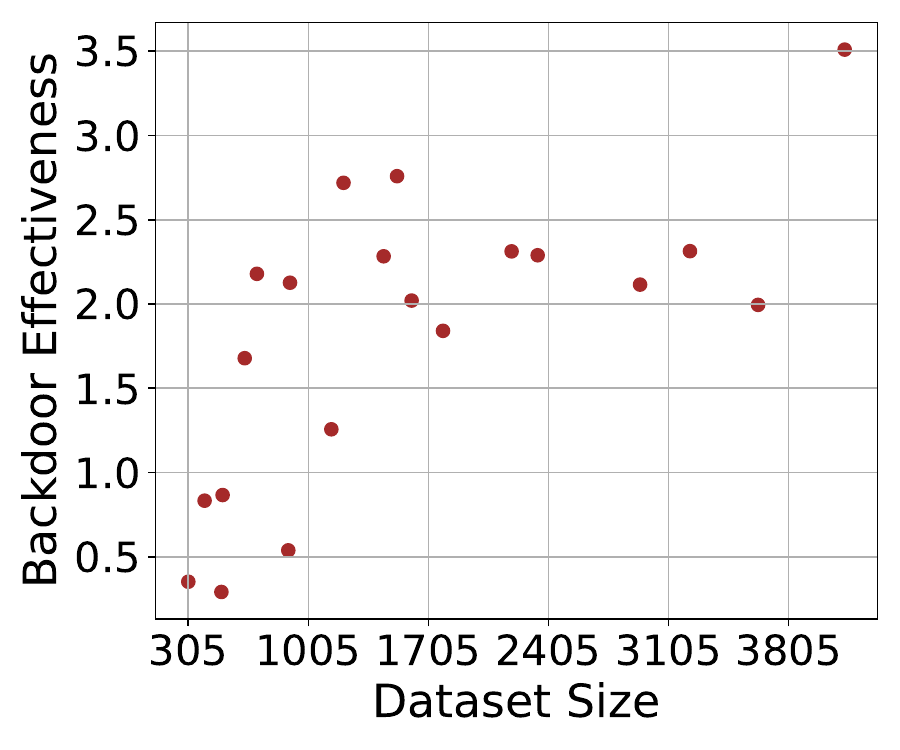}
        \subcaption*{B=6}
    \end{subfigure}
    \begin{subfigure}[t]{0.24\textwidth}
        \centering
        \includegraphics[width=\textwidth]{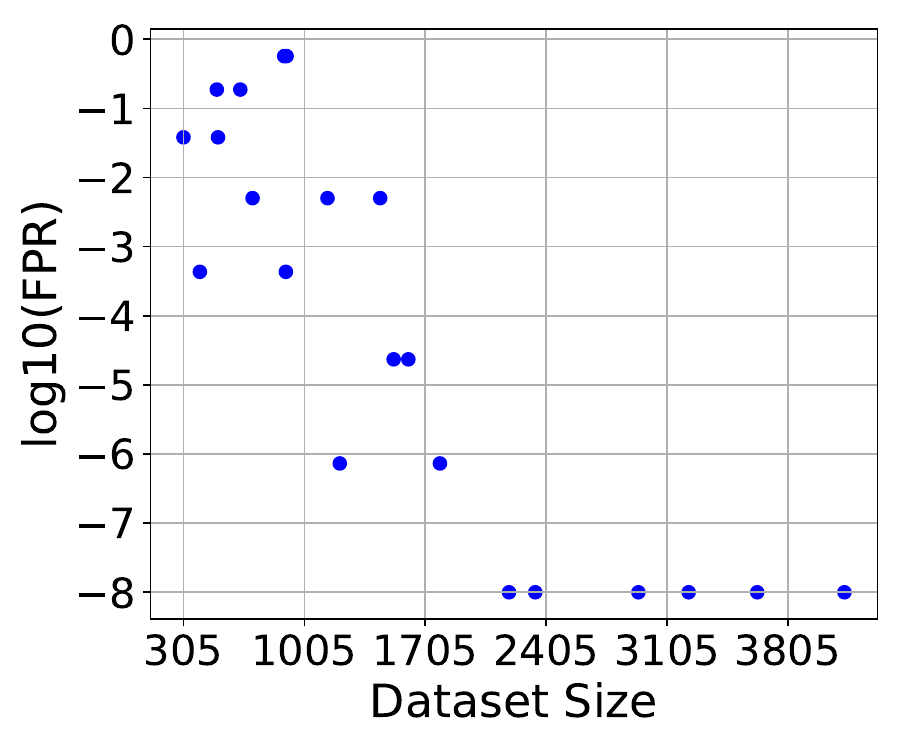}
        \includegraphics[width=\textwidth]{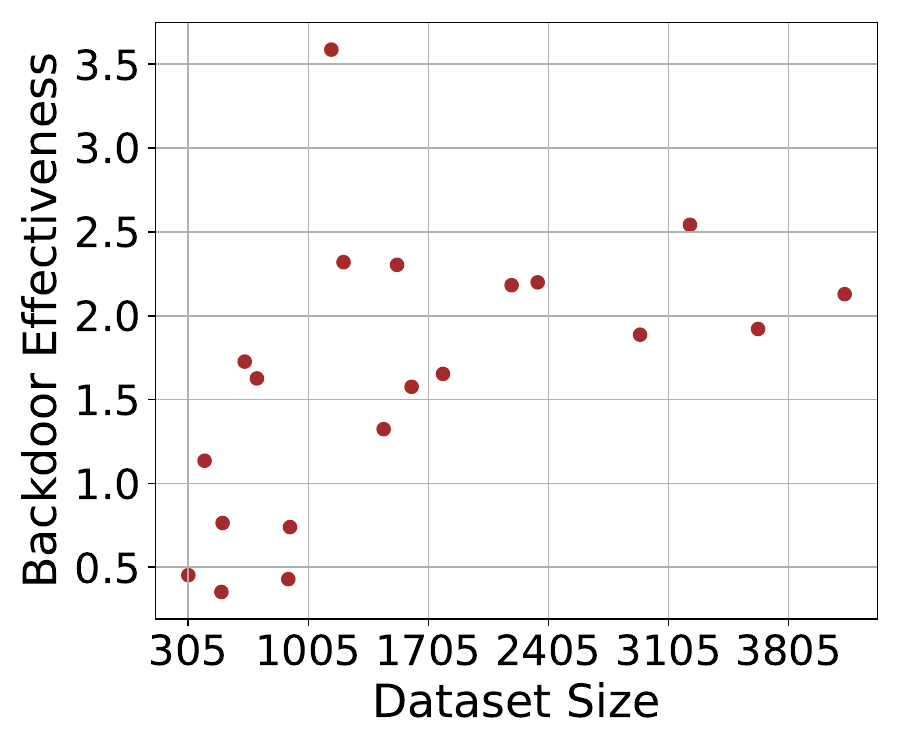}
        \subcaption*{B=8}
    \end{subfigure}

    \caption{The FPR for detecting contamination and the backdoor effectiveness as functions of the dataset size for Mistral-7B under different number of backdoors. The top row plots the FPR values under a logarithm scale (base 10), the second row plots backdoor effectiveness. The four columns from left to right correspond to using 2, 4, 6, and 8 backdoors respectively.} 
    \label{fig:mistral-fpr-size}
\end{figure*}
\begin{figure*}
    \centering

    \begin{subfigure}[t]{0.24\textwidth}
        \centering
        \includegraphics[width=\textwidth]{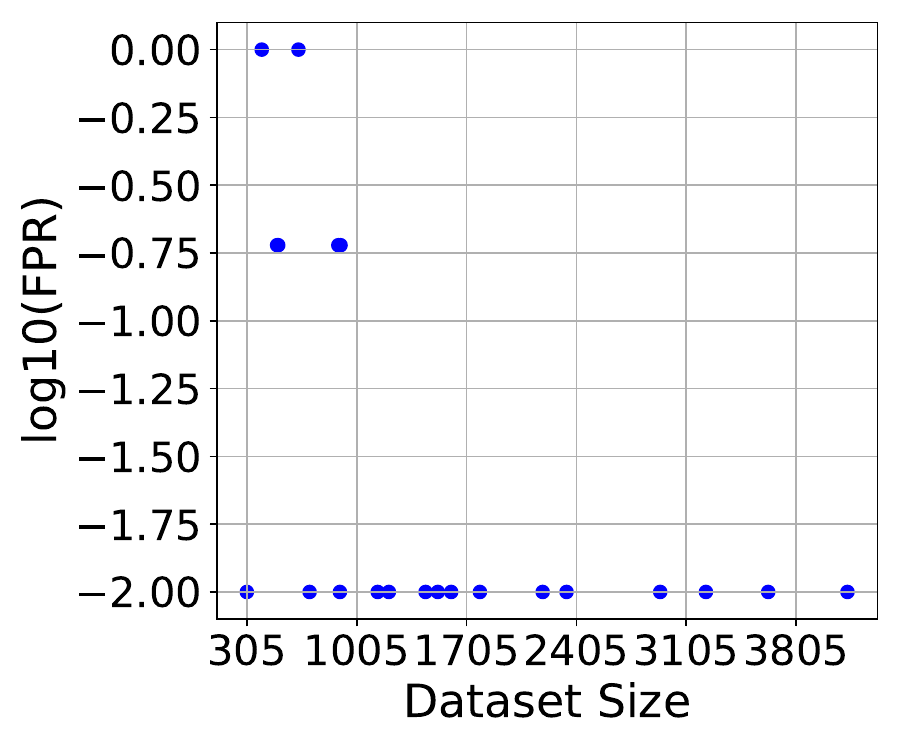}
        \includegraphics[width=\textwidth]{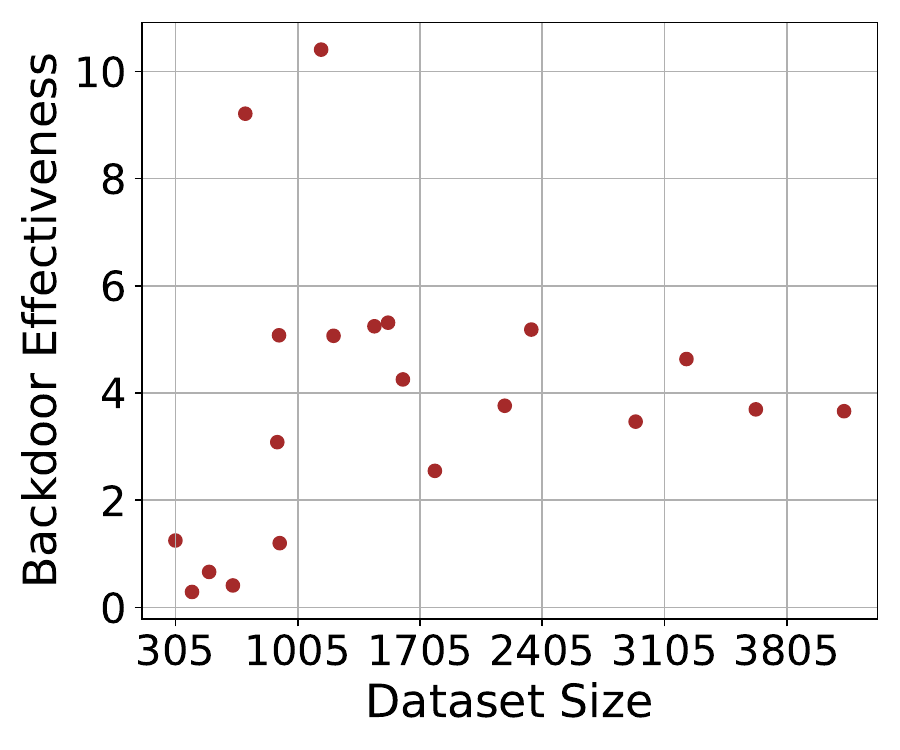}
        \subcaption*{B=2}
    \end{subfigure}
    \begin{subfigure}[t]{0.24\textwidth}
        \centering
        \includegraphics[width=\textwidth]{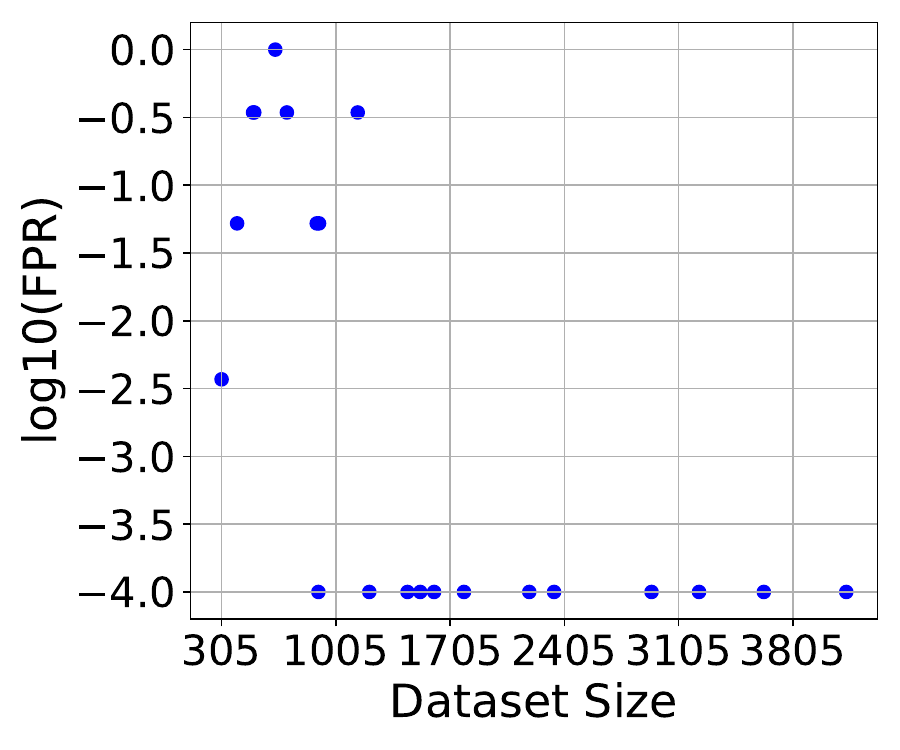}
        \includegraphics[width=\textwidth]{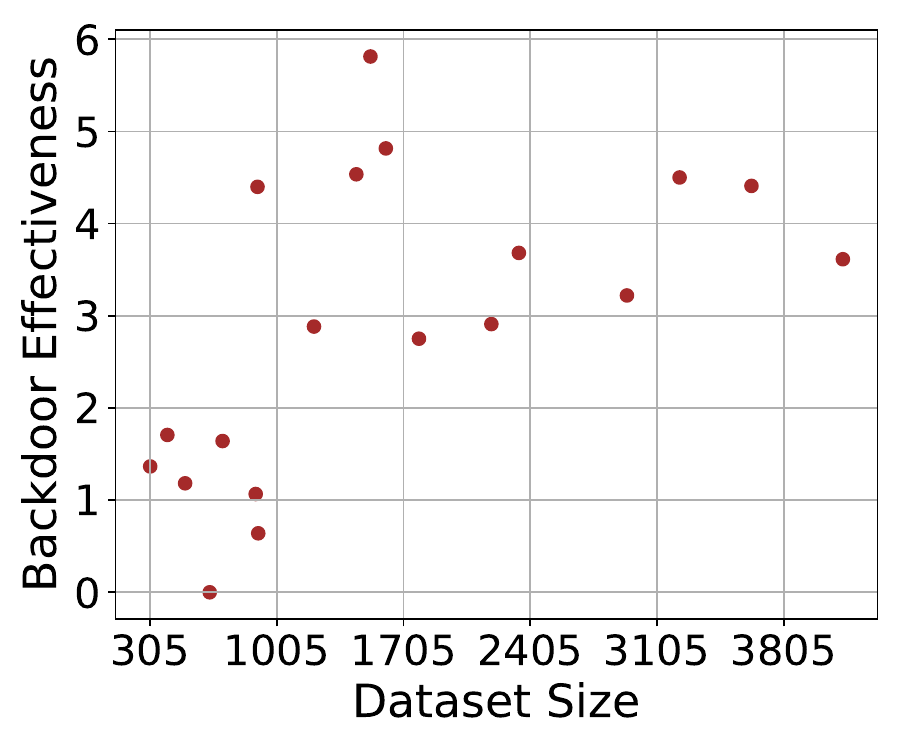}
        \subcaption*{B=4}
    \end{subfigure}
    \begin{subfigure}[t]{0.24\textwidth}
        \centering
        \includegraphics[width=\textwidth]{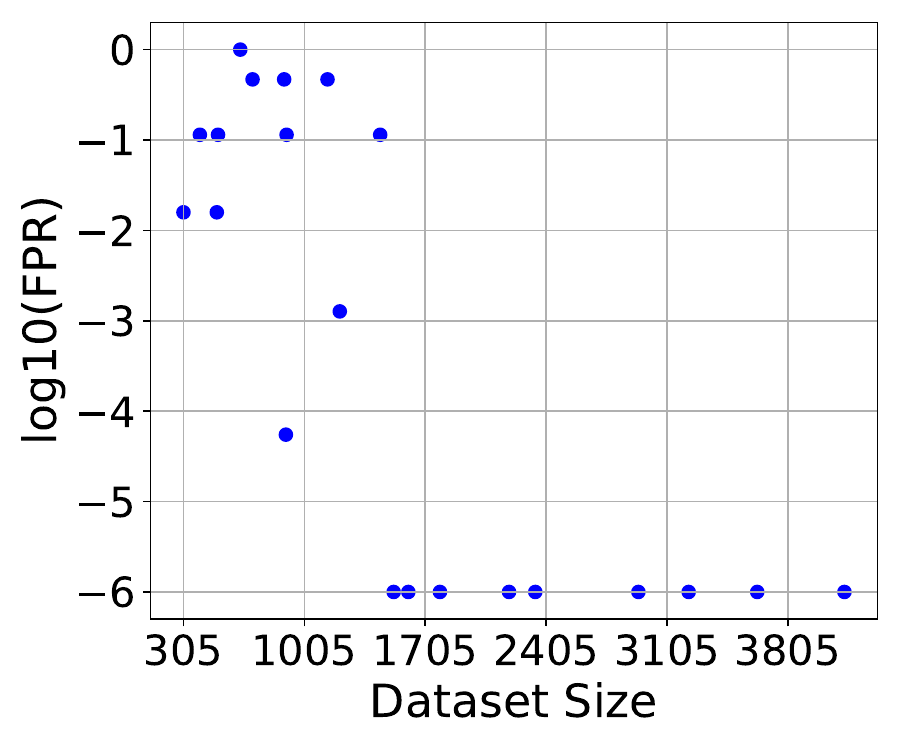}
        \includegraphics[width=\textwidth]{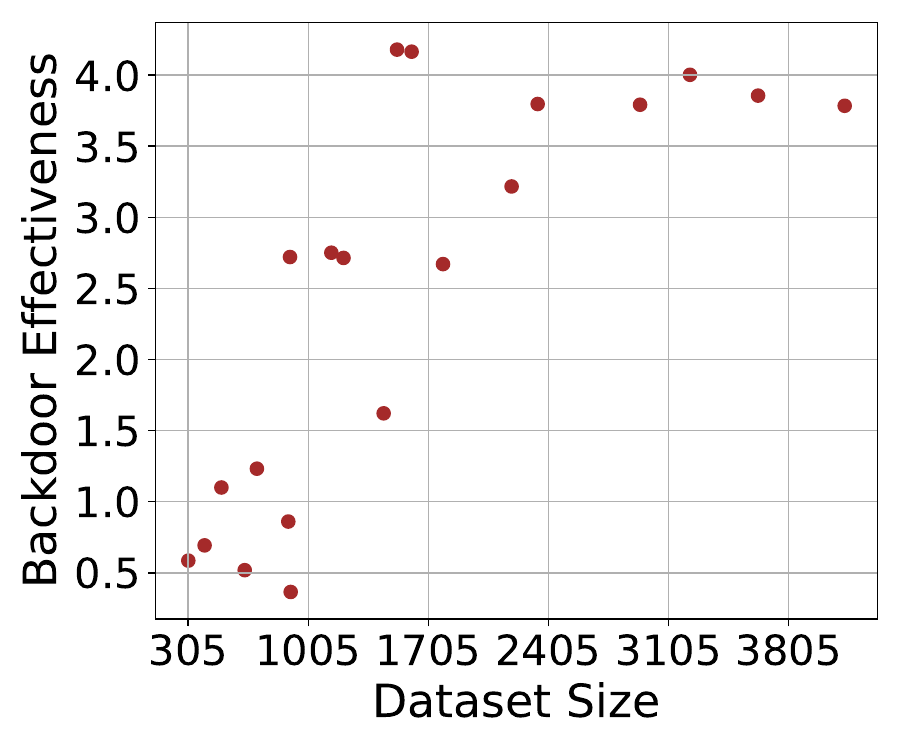}
        \subcaption*{B=6}
    \end{subfigure}
    \begin{subfigure}[t]{0.24\textwidth}
        \centering
        \includegraphics[width=\textwidth]{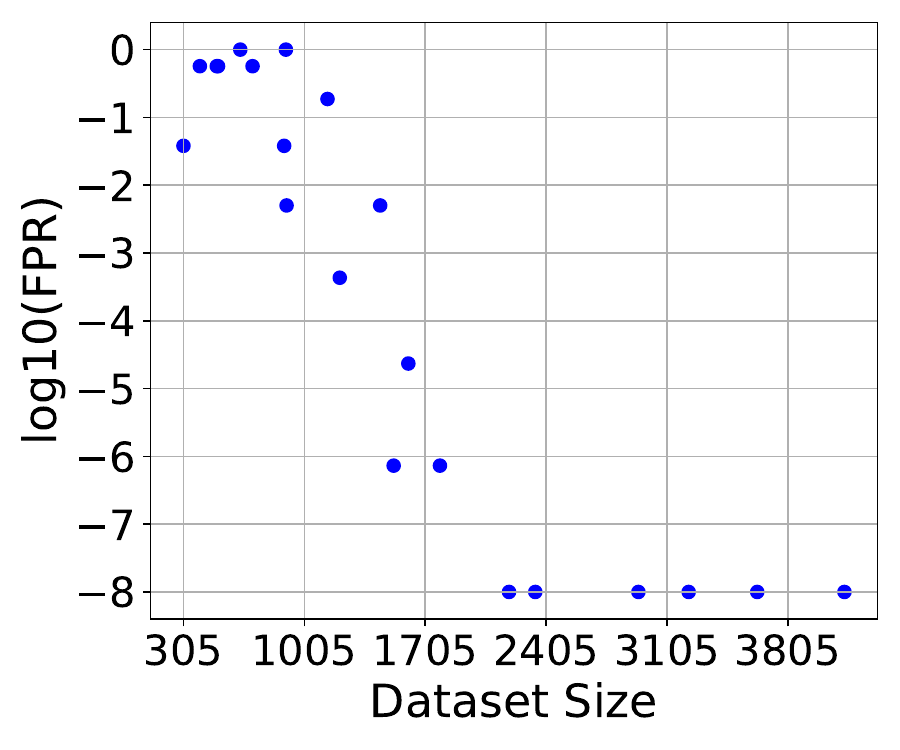}
        \includegraphics[width=\textwidth]{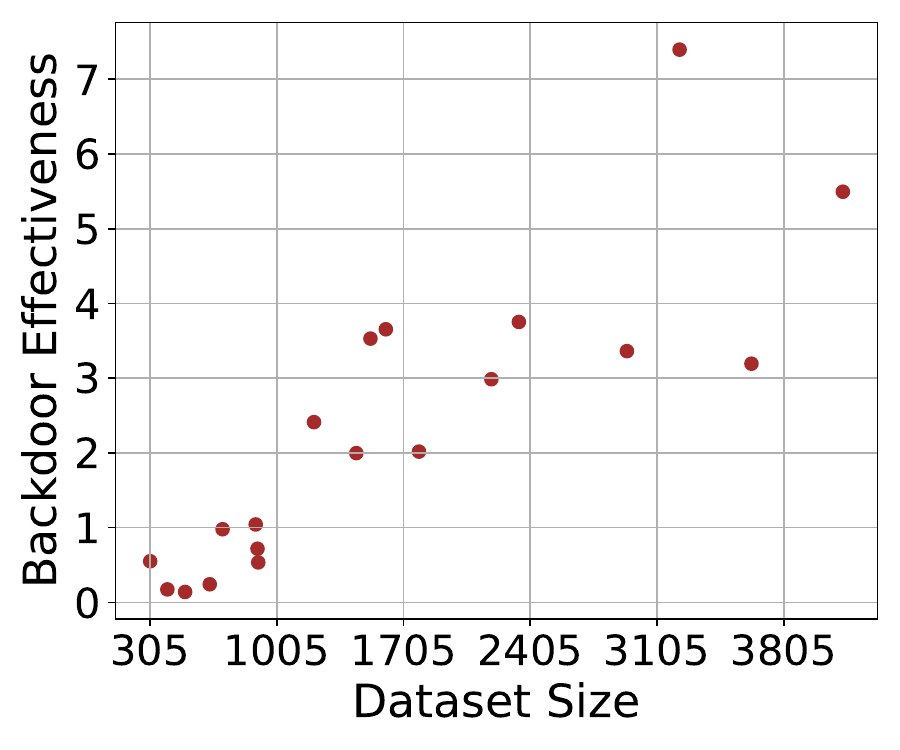}
        \subcaption*{B=8}
    \end{subfigure}

    \caption{The FPR for detecting contamination and the backdoor effectiveness as functions of the dataset size for Gemma-7B under different number of backdoors. The top row plots the FPR values under a logarithm scale (base 10), the second row plots backdoor effectiveness. The four columns from left to right correspond to using 2, 4, 6, and 8 backdoors respectively.} 
    \label{fig:gemma-fpr-size}
\end{figure*}

\section{Will Mixing Test Data with Backdoor Samples Undermine Evaluation Quality?}
\label{apped:eval_quality}

Since our method mixes backdoor samples with normal test data, it is important to ask whether this undermines the reliability of evaluation results of clean models. We argue that the effect is negligible, both in theory and in practice.

First, consider how clean models behave on backdoor samples. During test set preparation, as described in Section~\ref{subsec:test_set_prep}, the backdoor targets are randomly sampled from a uniform distribution $T_i \sim \text{Uniform}(1,K)$. Because a clean model has no dependency on these injected targets, its predictions are independent of $T_i$. Formally,
\[
T_i \mid f \overset{d}{=} T_i \sim \text{Uniform}(1,K),
\]
where $\overset{d}{=}$ denotes equality in distribution (The same conclusion was used in our proof of Theorem~\ref{thm: binomial}). This implies that clean models effectively guess on the injected samples, achieving an expected accuracy of $1/K$. As a consequence, no clean model gains a systematic advantage or disadvantage from the presence of the backdoor samples. This theoretical result is confirmed empirically in Appendix~\ref{append:main-result-score} (Table~\ref{tab:main-score}): for MMLU-Pro with $K=10$, most clean models achieve about $10\%$ accuracy on backdoor samples, while for Big-Bench-Hard with $K=7$, the accuracy fluctuates around $14.3\%$.

Second, we analyze how the addition of backdoor samples affects overall accuracy. Let $N$ denote the number of clean samples, $n_c$ the number of correct predictions on them, and $n_b$ the number of correct predictions on backdoor samples (using the backdoor targets as ground truth). Define a slightly modified version of poison rate\footnote{This differs slightly from the poison rate definition used elsewhere in our paper but simplifies the math without affecting conclusions.} as:
\[
p = \frac{\#\text{backdoor samples}}{\#\text{clean samples}}.
\]
The clean accuracy is $A_c = \tfrac{n_c}{N}$, while the combined accuracy is
\[
A_b = \frac{n_c + n_b}{(1+p)N}.
\]
Since $\mathbb{E}[n_b] = \tfrac{pN}{K}$, the relative difference between $A_b$ and $A_c$ is
\[
\epsilon = \frac{A_b - A_c}{A_c},
\]
Its expectation is\[
\mathbb{E}[\epsilon] = \Big(\tfrac{N/K}{n_c} - 1\Big)\cdot \tfrac{p}{1+p}.
\]
For any model performing better than random guess on clean data, the prefactor $\Big(\tfrac{N/K}{n_c} - 1\Big)$ lies strictly between $-1$ and $0$, which means that the accuracy distortion decreases on the order of $1/p$. And since the poison rate needed is rather small (as low as 2.2\% for our setup in Appendix~\ref{append:mix-data}, meaning we do not need to include too many backdoor samples), the relative error is negligible. In practice, the minimum poison rate required for backdoors to be effective depends on external factors outside of the DyePack framework—e.g., attack design, trigger strength, training hyperparameters. We clarify that our objective is not to propose a stronger backdoor attack method, but to theoretically and empirically demonstrate the effectiveness of repurposing backdoors for contamination detection while providing computable and bounded FPRs.

We also validate the stability of model ranking empirically. Table~\ref{tab:main-score} shows that across both datasets and 8 different values of $B$, the relative ranking of models remains unchanged before and after adding backdoor samples. For example, on MMLU-Pro with $B=8$, five models maintain exactly the same order despite small drops in raw accuracy. Across $100$ head-to-head model comparisons (two datasets, five values of $B$, and ten pairwise model combinations), the minimum injection rate required to flip any ranking is approximately $28.1\%$, which is far larger than the rates we used.

Moreover, in practice, when it comes to the need of strictly verifying the quality and trustworthiness of becnhmark evaluation, the more reliable and accepted approach is to use evaluator-run leaderboards (e.g., Open LLM Leaderboard~\cite{open-llm-leaderboard}, BFCL~\cite{patilberkeley}, LM Arena~\cite{lmarena_leaderboard}), rather than self-reported results (e.g., company blog posts). Since leaderboard owners run the evaluation, they know which samples are clean or backdoored, and can report accurate clean accuracy directly, which completely avoids any accuracy distortions caused by backdoor samples.

\end{document}